\documentclass[11pt]{article}
\usepackage{amsmath,amsbsy,amsfonts,amssymb,amsthm,dsfont,fullpage,units}
\usepackage{hyperref}
\usepackage[affil-sl]{authblk}
\usepackage[round]{natbib}
\newtheorem{theorem}{Theorem}
\newtheorem{lemma}[theorem]{Lemma}

\newtheorem{condition}{Condition}
\newtheorem{remark}[theorem]{Remark}
\newtheorem{proposition}[theorem]{Proposition}
\newtheorem{corollary}[theorem]{Corollary}

\def\abovestrut#1{\rule[0in]{0in}{#1}\ignorespaces}
\def\abovespace{\abovestrut{0.2in}}

\def\R{\mathbb{R}}
\def\E{\mathbb{E}}
\def\N{\mathbb{N}}
\def\wt{\widetilde}
\def\wh{\widehat}
\def\eps{\varepsilon}
\DeclareMathOperator{\diag}{diag}
\DeclareMathOperator{\range}{range}
\DeclareMathOperator{\rank}{rank}

\title{A Spectral Algorithm for Learning Hidden Markov Models}
\date{}

\author[1,2]{Daniel Hsu}
\author[2]{Sham M. Kakade}
\author[1]{Tong Zhang}
\affil[1]{Rutgers University, Piscataway, NJ 08854}
\affil[2]{University of Pennsylvania, Philadelphia, PA 19104}

\begin{document}
\maketitle

\begin{abstract}
Hidden Markov Models (HMMs) are one of the most fundamental and widely used
statistical tools for modeling discrete time series.
In general, learning HMMs from data is computationally hard (under
cryptographic assumptions), and practitioners typically resort to
search heuristics which suffer from the usual local optima issues.
We prove that under a natural separation condition (bounds on the smallest
singular value of the HMM parameters), there is an efficient and
provably correct algorithm for learning HMMs.
The sample complexity of the algorithm does not explicitly depend on the
number of distinct (discrete) observations---it implicitly depends on
this quantity through spectral properties of the underlying HMM.
This makes the algorithm particularly applicable to settings with a large
number of observations, such as those in natural language processing where
the space of observation is sometimes the words in a language.
The algorithm is also simple, employing only a singular value
decomposition and matrix multiplications.
\end{abstract}

\section{Introduction}

Hidden Markov Models (HMMs) \citep{BaumEagon67,Rabiner89} are the workhorse
statistical model for discrete time series, with widely diverse
applications including automatic speech recognition, natural language processing
(NLP), and genomic sequence modeling. In this model, a discrete hidden
state evolves according to some Markovian dynamics, and
observations at a particular time depend only on the hidden state at
that time. The learning problem is to estimate the
model only with observation samples from the underlying distribution. Thus far, the
predominant learning algorithms have been local search heuristics, such as
the Baum-Welch / EM algorithm \citep{BaumPSW70,EM}.

It is not surprising that practical algorithms have resorted to heuristics, as the general learning problem has been shown to be hard under cryptographic assumptions \citep{Terwijn02}. Fortunately, the hardness results are for HMMs that seem divorced from those that we are likely to encounter in practical applications. 

The situation is in many ways analogous to learning mixture distributions
with samples from the underlying distribution. There, the general problem is
also believed to be hard. However, much recent progress has been made when
certain separation assumptions are made with respect to the component
mixture distributions (\emph{e.g.}~\citep{Das99,DS07,VW02,CR08,BV08}). Roughly
speaking, these separation assumptions imply that with high
probability, given a point sampled from the distribution, one can determine
the mixture component that generated the point. In fact, there is a
prevalent sentiment that we are often only interested in clustering when such a separation condition holds. Much of the theoretical work
here has focused on how small this separation can be and still permit an efficient algorithm to recover the model.

We present a simple and efficient algorithm for learning HMMs under a certain
natural separation condition. We provide two results for learning. The
first is that we can approximate the joint distribution over observation
sequences of length $t$ (here, the quality of approximation is measured by
total variation distance). As $t$ increases,
the approximation quality degrades polynomially. Our second result is on
approximating the \emph{conditional} distribution over a future
observation, conditioned on some history of observations. We show that this
error is asymptotically bounded---\emph{i.e.}~for any $t$, conditioned on
the observations prior to time $t$, the error in predicting the $t$-th
outcome is controlled.
Our algorithm can be thought of as `improperly'
learning an HMM in that we do not explicitly recover the transition and
observation models.  However, our model does maintain a hidden state
representation which is closely (in fact, linearly) related to the HMM's,
and can be used for interpreting the hidden state.

The separation condition we require is a spectral condition on both the
observation matrix and the transition matrix. Roughly speaking, we require
that the observation distributions arising from distinct hidden states be
distinct (which we formalize by singular value conditions on the
observation matrix). This requirement can be thought of as being weaker than
the separation condition for clustering in that the observation distributions
can overlap quite a bit---given one observation, we do not necessarily have
the information to determine which hidden state it was generated from
(unlike in the clustering literature). We also have a spectral condition on
the correlation between adjacent observations. We believe both of these
conditions to be quite reasonable in many practical applications.
Furthermore, given our analysis, extensions to our algorithm which relax
these assumptions should be possible.

The algorithm we present has both polynomial sample and computational
complexity.  Computationally, the algorithm is quite simple---at its core
is a singular value decomposition (SVD) of a correlation matrix between
past and future observations. This SVD can be viewed as a
Canonical Correlation Analysis (CCA)~\citep{cca} between past and future
observations. The sample complexity results we present do not explicitly
depend on the number of distinct observations; rather, they implicitly
depend on this number through spectral properties of the HMM. This makes
the algorithm particularly applicable to settings with a large number of
observations, such as those in NLP where the space of observations is
sometimes the words in a language. 

\subsection{Related Work}

There are two ideas closely related to this work. The first comes from the
subspace identification literature in control
theory~\citep{Ljung,SubIDMoor,SubID}. The second idea is that, rather than
explicitly modeling the hidden states, we can represent the probabilities
of sequences of observations as products of matrix observation operators,
an idea which dates back to the literature on multiplicity
automata~\citep{S61,CP71,F74}.

The subspace identification methods, used in control theory, use spectral
approaches to discover the relationship between hidden states and the
observations. In this literature, the relationship is discovered for linear
dynamical systems such as Kalman filters. The basic idea is that the
relationship between observations and hidden states can often be discovered
by spectral/SVD methods correlating the past and future observations (in
particular, such methods often do a CCA between the past and future
observations). However, algorithms presented in the literature cannot be
directly used to learn HMMs because they assume additive noise models with
noise distributions independent of the underlying states, and such models
are not suitable for HMMs (an exception is~\citep{ARJ03}).
In our setting, we use this idea of performing a CCA between past and future observations to uncover information about the observation process (this is done through an SVD on a correlation matrix between past and future observations). The state-independent additive noise condition is avoided through the second idea.

The second idea is that we can represent the probability of sequences as
products of matrix operators, as in the literature on multiplicity
automata~\citep{S61,CP71,F74} (see \citep{pomdpMA} for discussion of this
relationship). This idea was re-used in both the Observable Operator Model
of \cite{jaeger} and the Predictive State Representations of \cite{LSS01},
both of which are closely related and both of which can model HMMs. In
fact, the former work by \cite{jaeger} provides a non-iterative algorithm
for learning HMMs, with an asymptotic analysis. However, this algorithm
assumed knowing a set of `characteristic events', which is a rather strong
assumption that effectively reveals some relationship between the hidden
states and observations. In our algorithm, this problem is avoided through
the first idea.

Some of the techniques in the work in \citep{Even-DarKM07} for tracking
belief states in an HMM are used here. As discussed earlier, we provide a
result showing how the model's conditional distributions over observations
(conditioned on a history) do not asymptotically diverge. This result was
proven in \citep{Even-DarKM07} when an approximate model is \emph{already
known}. Roughly speaking, the reason this error does not diverge is that
the previous observations are always revealing information about the next
observation; so with some appropriate contraction property, we would not
expect our errors to diverge. Our work borrows from this contraction
analysis.

Among recent efforts in various communities~\citep{ARJ03,VWM07,ZJ07,CC08},
the only previous efficient algorithm shown to PAC-learn HMMs in a setting
similar to ours is due to \cite{MR06}.
Their algorithm for HMMs is a specialization of a more general method for
learning phylogenetic trees from leaf observations.
While both this algorithm and ours rely on the same rank condition and
compute similar statistics, they differ in two significant regards.
First, \citep{MR06} were not concerned with large observation spaces, and
thus their algorithm assumes the state and observation spaces to have the
same dimension.
In addition,
\citep{MR06} take the more ambitious approach of learning the
observation and transition matrices explicitly, which unfortunately results
in a less sample-efficient algorithm that injects noise to
artificially spread apart the eigenspectrum of a probability matrix.
Our algorithm avoids recovering the observation and transition
matrix explicitly\footnote{In Appendix~\ref{section:mr06},
we discuss the key step in \citep{MR06}, and also
show how to use their technique in conjunction with our algorithm to
recover the HMM observation and transition matrices.
Our algorithm does not rely on this extra step---we believe it to be
generally unstable---but it can be taken if desired.}, and instead uses
subspace identification to learn an alternative
representation.

\section{Preliminaries} \label{section:preliminaries}

\subsection{Hidden Markov Models}

The HMM defines a probability distribution over sequences of hidden states
$(h_t)$ and observations $(x_t)$.
%
We write the set of hidden states as $[m] = \{ 1, \ldots, m \}$ and set
of observations as $[n] = \{ 1, \ldots, n \}$, where $m \leq n$.

Let $T \in \R^{m \times m}$ be the state transition
probability matrix with $T_{ij} = \Pr[h_{t+1} = i | h_t = j]$,
$O \in \R^{n \times m}$ be the observation
probability matrix with $O_{ij} = \Pr[x_t = i | h_t = j]$,
and $\vec{\pi} \in \R^m$ be the initial state distribution with $\vec{\pi}_i = \Pr[h_1 = i]$.
The conditional independence properties that an HMM satisfies are: 1)
conditioned on the previous hidden state, the current hidden state is sampled
independently of all other events in the history; and 2) conditioned on the
current hidden state, the current observation is sampled independently from
all other events in the history.
These conditional independence properties of the HMM imply that $T$ and $O$
fully characterize the probability distribution of any sequence of states
and observations.

A useful way of computing the probability of sequences is in terms of
`observation operators', an idea which dates back to the literature on
multiplicity automata (see~\citep{S61,CP71,F74}).  
The following lemma is straightforward to verify (see~\citep{jaeger,Even-DarKM07}).
\begin{lemma} \label{lemma:oom}
For $x = 1, \ldots, n$, define 
\[
A_x = T \diag(O_{x,1}, \ldots, O_{x,m}).
\]
For any $t$:
$$ \Pr[x_1, \ldots, x_t] \ = \ \vec{1}_m^\top A_{x_t} \ldots A_{x_1} \vec{\pi}. $$
\end{lemma}
Our algorithm learns a representation that is based on this observable
operator view of HMMs.

\subsection{Notation}
As already used in Lemma~\ref{lemma:oom}, the vector $\vec{1}_m$ is the
all-ones vector in $\R^m$.
We denote by $x_{1:t}$ the sequence $(x_1,\ldots,x_t)$, and by
$x_{t:1}$ its reverse $(x_t,\ldots,x_1)$. When we use a sequence as
a subscript, we mean the product of quantities indexed by the sequence elements.
So for example, the probability calculation in Lemma~\ref{lemma:oom} can be
written $\vec{1}_m^\top A_{x_{t:1}} \vec{\pi}$.
We will use $\vec{h}_t$ to denote a probability vector (a distribution
over hidden states),
with the arrow distinguishing it from the random hidden state variable
$h_t$.
Additional notation used in the theorem statements and proofs is listed in
Table~\ref{table:notation}.
%

\subsection{Assumptions}
We assume the HMM obeys the following condition.
\begin{condition}[HMM Rank Condition] \label{cond:rank}
$\vec{\pi} > 0$ element-wise, and $O$ and $T$ are rank $m$.
\end{condition}
The rank condition rules out the problematic case in which some state $i$
has an output distribution equal to a convex combination (mixture) of some
other states' output distributions.
Such a case could cause a learner to confuse state $i$ with a mixture of
these other states.
As mentioned before, the general task of learning HMMs (even the specific
goal of simply accurately modeling the distribution probabilities
\citep{Terwijn02}) is hard under cryptographic assumptions; the rank
condition is a natural way to exclude the malicious instances created by
the hardness reduction.

The rank condition on $O$ can be relaxed through a simple modification of
our algorithm that looks at multiple observation symbols simultaneously to
form the probability estimation tables.
For example, if two hidden states have identical observation probability in
$O$ but different transition probabilities in $T$, then they may be
differentiated by using two consecutive observations.
Although our analysis can be applied in this case with minimal
modifications, for clarity, we only state our results for an algorithm that
estimates probability tables with rows and columns corresponding to single
observations.

\subsection{Learning Model}
Our learning model is similar to those of~\citep{KMR+94,MR06} for
PAC-learning discrete probability distributions.
We assume we can sample observation sequences from an HMM.  In
particular, we assume each sequence is generated starting from the same
initial state distribution (\emph{e.g.}~the stationary distribution of the
Markov chain specified by $T$).  This setting is valid for practical
applications including speech recognition, natural language processing, and
DNA sequence modeling, where multiple independent sequences are available.

For simplicity, this paper only analyzes an algorithm that uses the initial 
few observations of each sequence, and ignores the rest.
We do this
to avoid using concentration bounds with complicated mixing conditions 
for Markov chains in our sample complexity calculation, 
as these conditions
are not essential to the main ideas we present. In practice, however,
one should use the full sequences to form the probability estimation tables
required by our algorithm.
In such scenarios, a single long sequence is sufficient for learning, and
the effective sample size can be simply discounted by the mixing rate of the
underlying Markov chain.

Our goal is to derive accurate estimators for the cumulative (joint)
distribution $\Pr[x_{1:t}]$ and the conditional distribution $\Pr[x_t |
x_{1:t-1}]$  for any sequence length $t$. For the conditional distribution,
we obtain an approximation that does not depend on $t$, while for the joint
distribution, the approximation quality degrades gracefully with $t$.

\section{Observable Representations of Hidden Markov Models}

A typical strategy for learning HMMs is to estimate the observation and
transition probabilities for each hidden state (say, by maximizing the
likelihood of a sample). However, since the hidden states are not directly
observed by the learner, one often resorts to heuristics (\emph{e.g.}~EM) that
alternate between imputing the hidden states and selecting parameters $\wh O$
and $\wh T$ that maximize the likelihood of the sample and current state
estimates. Such heuristics can suffer from local optima issues and require
careful initialization (\emph{e.g.}~an accurate guess of the hidden states) to
avoid failure.

However, under Condition~\ref{cond:rank}, HMMs admit an efficiently learnable
parameterization that depends only on \emph{observable quantities}. Because
such quantities can be estimated from data, learning this representation
avoids any guesswork about the hidden states and thus allows for algorithms
with strong guarantees of success.

This parameterization is natural in the context of Observable Operator
Models~\citep{jaeger}, but here we emphasize its connection to subspace
identification.

\subsection{Definition}

Our HMM representation is defined in terms of the following
vector and matrix quantities:
\begin{eqnarray*}
{[P_1]}_i & = & \Pr[x_1 = i] \\
{[P_{2,1}]}_{ij} & = & \Pr[x_2=i,x_1=j] \\
{[P_{3,x,1}]}_{ij} & = & \Pr[x_3=i,x_2=x,x_1=j]
\quad \forall x \in [n],
\end{eqnarray*}
where $P_1 \in \R^n$ is a vector, and $P_{2,1} \in \R^{n \times n}$
and the $P_{3,x,1} \in \R^{n \times n}$
are matrices.
These are the marginal probabilities of observation singletons, pairs, and
triples. 

The representation further depends on a matrix $U \in \R^{n \times m}$ that
obeys the following condition.
\begin{condition}[Invertibility Condition] \label{cond:U}
$U^\top O$ is invertible.
\end{condition}
In other words, $U$ defines an $m$-dimensional subspace that preserves the
state dynamics---this will become evident in the next few lemmas.

A natural choice for $U$ is given by the `thin' SVD of $P_{2,1}$, as the
next lemma exhibits.
\begin{lemma} \label{lemma:svdU}
Assume $\vec{\pi} > 0$ and that $O$ and $T$ have column rank $m$. Then
$\rank(P_{2,1}) = m$. Moreover, if $U$ is the matrix of left singular
vectors of $P_{2,1}$ corresponding to non-zero singular values, then
$\range(U) = \range(O)$, so $U \in \R^{n \times m}$ obeys
Condition~\ref{cond:U}.
\end{lemma}

\begin{proof}
Using the conditional independence properties of the HMM,
entries of the matrix $P_{2,1}$ can be factored as
\begin{align*}
[P_{2,1}]_{ij}
& = \sum_{k=1}^m \sum_{\ell=1}^m
\Pr[x_2 = i, x_1 = j, h_2 = k, h_1 = \ell] \\
& = \sum_{k=1}^m \sum_{\ell=1}^m
O_{ik} \ T_{k\ell} \ \vec{\pi}_{\ell} \ [O^\top]_{\ell j}
\end{align*}
so $P_{2,1} = OT \diag(\vec{\pi}) O^\top$ and thus $\range(P_{2,1}) \subseteq
\range(O)$.
The assumptions on $O$, $T$, and $\vec{\pi}$ imply that $T \diag(\vec{\pi}) O^\top$
has linearly independent rows and that $P_{2,1}$ has $m$ non-zero singular values.
Therefore
$$ O = P_{2,1} (T \diag(\vec{\pi}) O^\top)^+ $$
(where $X^+$ denotes the Moore-Penrose pseudo-inverse of a matrix $X$~\citep{perturbation}),
which in turn implies $\range(O) \subseteq \range(P_{2,1})$.
Thus $\rank(P_{2,1}) = \rank(O) = m$, and also $\range(U) = \range(P_{2,1})
= \range(O)$.
\end{proof}

Our algorithm is motivated by Lemma~\ref{lemma:svdU} in that we compute the SVD
of an empirical estimate of $P_{2,1}$ to discover a $U$ that satisfies
Condition~\ref{cond:U}.
We also note that this choice for $U$ can be thought of as a surrogate for
the observation matrix $O$ (see Remark~\ref{remark:U}).

Now given such a matrix $U$, we can finally define the observable representation:
\begin{eqnarray*}
\vec{b}_1 & = & U^\top P_1 \\
\vec{b}_\infty & = & \left( P_{2,1}^\top U \right)^+ P_1 \\
B_x & = & \left( U^\top P_{3,x,1} \right)
\left( U^\top P_{2,1} \right)^+ \quad \forall x \in [n] \ .
\end{eqnarray*}

\subsection{Basic Properties}

The following lemma shows that the observable representation,
parameterized by $\{ \vec{b}_\infty, \vec{b}_1, B_1, \ldots, B_n \}$, is
sufficient to compute the probabilities of any sequence of observations.
\begin{lemma}[Observable HMM Representation] \label{lemma:obs}
Assume the HMM obeys Condition~\ref{cond:rank}
and that $U \in \R^{n \times m}$ obeys Condition~\ref{cond:U}.
Then:
\begin{enumerate}
\item $\vec{b}_1 \ = \ (U^\top O) \vec{\pi}$.
\item $\vec{b}_\infty^\top \ = \ \vec{1}_m^\top (U^\top O)^{-1}$.
\item $B_x \ = \ (U^\top O) A_x (U^\top O)^{-1}$ \ $\forall x \in [n]$.
\item $\Pr[x_{1:t}]
\ = \ \vec{b}_\infty^\top B_{x_{t:1}} \vec{b}_1$
\ $\forall t \in \N, x_1,\ldots,x_t \in [n]$.
\end{enumerate}
\end{lemma}

In addition to joint probabilities, we can compute conditional
probabilities using the observable representation.
We do so through (normalized) conditional `internal states' that depend on a
history of observations. We should emphasize that these states are \emph{not} in fact probability distributions over hidden states (though the following lemma shows that they are linearly related).
As per Lemma~\ref{lemma:obs}, the initial state is
$$ \vec{b}_1 \ =  \ (U^\top O) \vec{\pi}. $$
Generally, for any $t \geq 1$, given observations $x_{1:t-1}$ with
$\Pr[x_{1:t-1}] > 0$, we define the internal state as:
$$ \vec{b}_t \ = \ \vec{b}_t(x_{1:t-1}) \ = \
\frac{B_{x_{t-1:1}} \vec{b}_1}{\vec{b}_\infty^\top B_{x_{t-1:1}} \vec{b}_1}. $$
The case $t = 1$ is consistent with the general definition of $\vec{b}_t$ because
the denominator is $\vec{b}_\infty^\top \vec{b}_1 = \vec{1}_m^\top (U^\top O)^{-1} (U^\top O)
\vec{\pi}
= \vec{1}_m^\top \vec{\pi} = 1$.
The following result shows how these internal states can be used to
compute conditional probabilities $\Pr[x_t=i|x_{1:t-1}]$.

\begin{lemma}[Conditional Internal States] \label{lemma:belief}
Assume the conditions in Lemma~\ref{lemma:obs}.
Then, for any time $t$:
\begin{enumerate}
\item (Recursive update of states)
If $\Pr[x_{1:t}] > 0$, then
$$ \vec{b}_{t+1} \ = \ \frac{B_{x_t} \vec{b}_t}{\vec{b}_\infty^\top B_{x_t}
\vec{b}_t}, $$
\item (Relation to hidden states)
$$ \vec{b}_t \ = \ (U^\top O) \ \vec{h}_t(x_{1:t-1})$$
where $[\vec{h}_t(x_{1:t-1})]_i \ = \ \Pr[h_t = i | x_{1:t-1}]$
is the conditional probability of the hidden state at time $t$ given the
observations $x_{1:t-1}$,

\item (Conditional observation probabilities)
$$   \Pr[x_t | x_{1:t-1}] \ = \  \vec{b}_\infty^\top B_{x_t}
\vec{b}_t. $$
\end{enumerate}
\end{lemma}

\begin{remark} \label{remark:U}
If $U$ is the matrix of left singular vectors of $P_{2,1}$ corresponding to
non-zero singular values, then $U$ acts much like the observation
probability matrix $O$ in the following sense:
\begin{center} \begin{tabular}{cc}
\begin{minipage}{0.4\textwidth} \begin{center}
Given a conditional state $\vec{b}_t$, \\
$\Pr[x_t = i | x_{1:t-1}] = [U\vec{b}_t]_i$.
\end{center} \end{minipage}
&
\begin{minipage}{0.4\textwidth} \begin{center}
Given a conditional hidden state $\vec{h}_t$, \\
$\Pr[x_t = i | x_{1:t-1}] = [O\vec{h}_t]_i$.
\end{center} \end{minipage}
\end{tabular} \end{center}
To see this, note that $UU^\top$ is the projection
operator to $\range(U)$. Since $\range(U) = \range(O)$
(Lemma~\ref{lemma:svdU}), we have $UU^\top O = O$, so
$U\vec{b}_t = U (U^\top O) \vec{h}_t = O\vec{h}_t$.
\end{remark}

\subsection{Proofs}

\begin{proof}[Proof of Lemma~\ref{lemma:obs}]
The first claim is immediate from the fact $P_1 = O \vec{\pi}$.
For the second claim, we write $P_1$ in the following unusual (but easily
verified) form:
\begin{eqnarray*}
P_1^\top & = & \vec{1}_m^\top T \diag(\vec{\pi}) O^\top \\
& = & \vec{1}_m^\top (U^\top O)^{-1} (U^\top O) T \diag(\vec{\pi}) O^\top \\
& = & \vec{1}_m^\top (U^\top O)^{-1} U^\top P_{2,1}.
\end{eqnarray*}
The matrix $U^\top P_{2,1}$ has linearly independent rows (by the
assumptions on $\vec{\pi}$, $O$, $T$, and the condition on $U$), so
$$ \vec{b}_\infty^\top
\ = \ P_1^\top (U^\top P_{2,1})^+
\ = \ \vec{1}_m^\top (U^\top O)^{-1} \ (U^\top P_{2,1}) \ (U^\top P_{2,1})^+
\ = \ \vec{1}_m^\top (U^\top O)^{-1}.
$$
To prove the third claim, we first express $P_{3,x,1}$ in terms of $A_x$:
\begin{eqnarray*}
P_{3,x,1}
& = & O A_x T \diag(\vec{\pi}) O^\top \\
& = & O A_x (U^\top O)^{-1} (U^\top O) T \diag(\vec{\pi}) O^\top \\
& = & O A_x (U^\top O)^{-1} U^\top P_{2,1}.
\end{eqnarray*}
Again, using the fact that $U^\top P_{2,1}$ has full row rank,
\begin{eqnarray*}
B_x
& = & \left( U^\top P_{3,x,1} \right) \ \left( U^\top P_{2,1} \right)^+ \\
& = & (U^\top O) A_x (U^\top O)^{-1} \
\left( U^\top P_{2,1} \right) \ \left( U^\top P_{2,1} \right)^+ \\
& = & (U^\top O) A_x (U^\top O)^{-1}.
\end{eqnarray*}
The probability calculation in the fourth claim is now readily seen as a
telescoping product that reduces to the product in Lemma~\ref{lemma:oom}.
\end{proof}

\begin{proof}[Proof of Lemma~\ref{lemma:belief}]
The first claim is a simple induction.
The second and third claims are also proved by induction as follows.
The base case is clear from Lemma~\ref{lemma:obs} since $\vec{h}_1 = \vec{\pi}$ and
$\vec{b}_1 = (U^\top O) \vec{\pi}$, and also $\vec{b}_\infty^\top B_{x_1}
\vec{b}_1 = \vec{1}_m^\top
A_{x_1} \vec{\pi} = \Pr[x_1]$.
For the inductive step,
\begin{eqnarray*}
\vec{b}_{t+1}
& = & \frac{B_{x_t} \vec{b}_t}{\vec{b}_\infty^\top B_{x_t} \vec{b}_t}
\\
& = & \frac{B_{x_t} (U^\top O) \vec{h}_t}
{\Pr[x_t|x_{1:t-1}]}
\quad \text{(inductive hypothesis)} \\
& = & \frac{(U^\top O) A_{x_t} \vec{h}_t}
{\Pr[x_t|x_{1:t-1}]}
\quad \text{(Lemma~\ref{lemma:obs})} \\
& = & (U^\top O) \
\frac {\Pr[h_{t+1} = \cdot, x_t | x_{1:t-1}]} {\Pr[x_t | x_{1:t-1}]}
\\
& = & (U^\top O) \
\frac{\Pr[h_{t+1} = \cdot | x_{1:t}] \Pr[x_t | x_{1:t-1}]}
{\Pr[x_t | x_{1:t-1}]}
\\
& = & (U^\top O) \ \vec{h}_{t+1}(x_{1:t})
\end{eqnarray*}
and
$$ \vec{b}_\infty^\top B_{x_{t+1}} \vec{b}_{t+1} \ = \ \vec{1}_m^\top
A_{x_{t+1}} \vec{h}_{t+1} \ = \ \Pr[x_{t+1} | x_{1:t}] $$
(again, using Lemma~\ref{lemma:obs}).
\end{proof}

\section{Spectral Learning of Hidden Markov Models}

\subsection{Algorithm}

The representation in the previous section suggests the algorithm
detailed in Figure~\ref{fig:alg}, which simply
uses random samples to estimate the model parameters. Note that in
practice, knowing $m$ is not essential because the method presented here
tolerates models that are not exactly HMMs, and the parameter $m$ may be
tuned using cross-validation. As we discussed earlier, the requirement for
independent samples is only for the convenience of our sample complexity
analysis.

\begin{figure}[top]
\framebox[\textwidth]{\begin{minipage}{0.9\textwidth}
\vskip 0.1in
\underline{Algorithm \ \textsc{LearnHMM$(m,N)$}}: \\
Inputs: $m$ - number of states, $N$ - sample size \\
Returns: HMM model parameterized by $\{ \wh b_1, \wh b_\infty, \wh B_x \
\forall x \in [n] \}$
\begin{enumerate}
\item Independently sample $N$ observation triples $(x_1,x_2,x_3)$ from the HMM
to form empirical estimates $\wh P_1, \wh P_{2,1}, \wh P_{3,x,1}$ $\forall
x \in [n]$ of $P_1, P_{2,1}, P_{3,x,1}$ $\forall x \in [n]$.

\item Compute the SVD of $\wh P_{2,1}$, and let $\wh U$ be the matrix of
left singular vectors corresponding to the $m$ largest singular values.

\item Compute model parameters:
\begin{enumerate}
\item $\wh b_1 = \wh U^\top \wh P_1$,
\item $\wh b_\infty = (\wh P_{2,1}^\top \wh U)^+ P_1$,
\item $\wh B_x = \wh U^\top \wh P_{3,x,1} (\wh U^\top \wh P_{2,1})^+$
$\forall x \in [n]$.
\end{enumerate}
\end{enumerate}
\vskip 0.1in
\end{minipage}}
\caption{HMM learning algorithm.}
\label{fig:alg}
\end{figure}

The model returned
by \textsc{LearnHMM$(m,N)$} can be used as follows:
\begin{itemize}
\item To predict the probability of a sequence:
$$ \wh\Pr[ x_1, \ldots, x_t ]
\ = \ \wh b_\infty^\top \wh B_{x_t} \ldots \wh B_{x_1} \wh b_1. $$

\item Given an observation $x_t$, the `internal state' update is:
$$ \wh b_{t+1}
\ = \ \frac{\wh B_{x_t} \wh b_t}{\wh b_\infty^\top \wh B_{x_t} \wh b_t}. $$ 

\item To predict the conditional probability of $x_t$ given $x_{1:t-1}$:
$$ \wh\Pr[ x_t | x_{1:t-1} ]
\ = \ \frac
{\wh b_\infty^\top \wh B_{x_t} \wh b_t}
{\sum_x \wh b_\infty^\top \wh B_x \wh b_t}. $$
\end{itemize}

Aside from the random sampling, the running time of the learning algorithm
is dominated by the SVD computation of an $n \times n$ matrix.
The time required for computing joint probability calculations is $O(tm^2)$
for length $t$ sequences---same as if one used the ordinary HMM
parameters ($O$ and $T$).
For conditional probabilities, we require some extra work (proportional to
$n$) to compute the normalization factor.
However, our analysis shows that this normalization factor is always close
to $1$ (see Lemma~\ref{lemma:approx-update}), so it can be safely omitted
in many applications.

%

Note that the algorithm does not explicitly ensure that the predicted
probabilities lie in the range $[0,1]$.
This is a dreaded problem that has been faced by other methods for learning
and using general operator models~\cite{jaeger}, and a number of heuristic
for coping with the problem have been proposed and may be applicable here
(see~\cite{jaeger2} for some recent developments).
We briefly mention that in the case of joint probability prediction,
clipping the predictions to the interval $[0,1]$ can only increase the
$L_1$ accuracy, and that the KL accuracy guarantee explicitly requires the
predicted probabilities to be non-zero.

\subsection{Main Results}

We now present our main results.
The first result is a guarantee on the accuracy of our joint probability
estimates for observation sequences.
The second result concerns the accuracy of conditional probability estimates --- a much more delicate quantity to bound due to conditioning on unlikely events.
We also remark that if the probability distribution is only approximately
modeled as an HMM, then our results degrade gracefully based on this
approximation quality.

\subsubsection{Joint Probability Accuracy}

Let $\sigma_m(M)$ denote the $m$th largest singular value of a matrix $M$.
Our sample complexity bound will depend polynomially on
$1/\sigma_m(P_{2,1})$ and $1/\sigma_m(O)$.

%
Also, define
\begin{equation}
\epsilon(k) \ = \ \min \left\{ \sum_{j \in S} \Pr[x_2=j] : S \subseteq
[n], |S| = n - k \right\},
\label{eq:epsk}
\end{equation}
and let
$$ n_0(\eps) \ = \ \min \{ k : \epsilon(k) \leq \eps \}.$$
In other words, $n_0(\eps)$ is the minimum number of observations that
account for about $1 - \epsilon$ of the total probability mass.
Clearly $n_0(\eps) \leq n$, but it can often be much smaller in real
applications.
For example, in many practical applications, the frequencies of observation
symbols observe a power law (called Zipf's law) of the form $f(k) \propto
1/k^s$, where $f(k)$ is the frequency of the $k$-th most frequently
observed symbol.
If $s>1$, then $\epsilon(k)= O(k^{1-s})$, and $n_0(\eps)=O(\eps^{1/(1-s)})$
becomes independent of the number of observations $n$.
This means that for such problems, our analysis below leads to a sample
complexity bound for the cumulative distribution $\Pr[x_{1:t}]$ that can
be independent of $n$. This is useful in domains with large $n$ such as
natural language processing.

\begin{theorem} \label{theorem:main-l1}
There exists a constant $C > 0$ such that the following holds.
Pick any $0 < \epsilon, \eta < 1$ and $t \geq 1$,
and let $\eps_0 = \sigma_m(O) \sigma_m(P_{2,1}) \epsilon / (4t\sqrt m)$.
Assume the HMM obeys Condition~\ref{cond:rank}, and
$$ N \geq C \cdot
\frac{t^2}{\epsilon^2}
\cdot
\left(
\frac{m}{\sigma_m(O)^2\sigma_m(P_{2,1})^4}
+ \frac{m \cdot n_0(\eps_0)}{\sigma_m(O)^2\sigma_m(P_{2,1})^2}
\right) \cdot \log \frac{1}{\eta}
. $$
With probability at least $1 - \eta$,
the model returned by the algorithm \textsc{LearnHMM$(m,N)$} satisfies
$$ \sum_{x_1,\ldots,x_t} | \Pr[x_1,\ldots,x_t] - \wh
\Pr[x_1,\ldots,x_t]| \leq \epsilon. $$
\end{theorem}
The main challenge in proving Theorem~\ref{theorem:main-l1} is
understanding how the estimation errors accumulate in the algorithm's
probability calculation. This would have been less problematic if we had
estimates of the usual HMM parameters $T$ and $O$; the fully observable
representation forces us to deal with more cumbersome matrix and vector
products.

\subsubsection{Conditional Probability Accuracy}

In this section, we analyze the accuracy of our conditional probability predictions
$\wh \Pr[x_t | x_1, \ldots, x_{t-1}]$. Intuitively, we might hope that
these predictive distributions do not become arbitrarily bad over time,
(as
$t\rightarrow \infty$). The reason is that while estimation errors
propagate into long-term probability predictions (as evident in
Theorem~\ref{theorem:main-l1}), the history of observations constantly provides
feedback about the underlying hidden state, and this information is
incorporated using Bayes' rule (implicitly via our internal state updates).

This intuition was confirmed by \cite{Even-DarKM07}, who showed that if one
has an approximate model of $T$ and $O$ for the HMM, then under certain
conditions, the conditional prediction does not diverge.
This condition is the positivity of the `value
of observation' $\gamma$, defined as
$$ \gamma \ = \ \inf_{\vec{v}:\|\vec{v}\|_1 =1} \|O\vec{v}\|_1. $$
Note that $\gamma \geq \sigma_m(O)/\sqrt{n}$, so it is guaranteed to be
positive by Condition~\ref{cond:rank}.
However, $\gamma$ can be much larger than what this crude lower bound suggests.

To interpret this quantity $\gamma$, consider any two distributions over hidden
states $\vec{h}, \wh h \in \R^m$. Then $\|O(\vec{h}-\wh h)\|_1 \geq \gamma
\|\vec{h}-\wh h\|_1$.
Regarding $\vec{h}$ as the true hidden state distribution and $\wh h$ as the estimated hidden
state distribution, this inequality gives a lower bound on the error of the estimated
observation distributions under $O$. In other words, the observation process, on average, reveal errors in our hidden state estimation.
The work of \citep{Even-DarKM07} uses this as a contraction property to show how
prediction errors (due to using an approximate model) do not diverge.
In our setting, this is more difficult as we do not explicitly estimate $O$ nor do we explicitly maintain distributions over hidden states.

We also need the following assumption, which we discuss further
following the theorem statement. 
\begin{condition}[Stochasticity Condition] \label{cond:alpha}
For all observations $x$ and all states $i$ and $j$, $[A_x]_{ij} \geq
\alpha > 0$.
\end{condition}

\begin{theorem} \label{theorem:main-tracking}
There exists a constant $C > 0$ such that the following holds.
Pick any $0 < \epsilon, \eta < 1$,
and let $\eps_0 = \sigma_m(O) \sigma_m(P_{2,1}) \epsilon / (4\sqrt m)$.
Assume the HMM obeys Conditions~\ref{cond:rank} and \ref{cond:alpha}, and
$$ N \ \geq \ C \cdot
\left(
\left(
\frac{m}{\epsilon^2 \alpha^2}
+ \frac{(\log (2/\alpha))^4}{\epsilon^4 \alpha^2 \gamma^4}
\right)
\cdot
\frac{m}{\sigma_m(O)^2 \sigma_m(P_{2,1})^4}
+
\frac{1}{\epsilon^2}
\cdot
\frac{m \cdot n_0(\eps_0)}{\sigma_m(O)^2 \sigma_m(P_{2,1})^2}
\right)
\cdot
\log \frac{1}{\eta} .
$$
With probability at least $1 - \eta$, then the model returned by
\textsc{LearnHMM$(m,N)$} satisfies, for any time $t$,
$$
KL(\Pr[x_t | x_1, \ldots, x_{t-1}] \ || \ \wh \Pr[x_t | x_1, \ldots, x_{t-1}])
\ = \
\E_{x_{1:t}} \left[ 
\ln \frac{ \Pr[ x_t | x_{1:t-1} ] }{ \wh\Pr[ x_t | x_{1:t-1} ]} \right]
\ \leq \ \epsilon. 
$$
\end{theorem}
To justify our choice of error measure, note that
the problem of bounding the errors of conditional probabilities is
complicated by the issue of that, over the long run, we may have to condition on a very low probability event.
Thus we need to control the relative accuracy of our predictions.
This makes the KL-divergence a natural choice for
the error measure. Unfortunately, because our HMM conditions are more naturally
interpreted in terms of spectral and normed quantities, we end up switching
back and forth between KL and $L_1$ errors via Pinsker-style inequalities
(as in \citep{Even-DarKM07}). It is not clear to us if a significantly better guarantee could be obtained with a pure $L_1$ error analysis (nor is it clear how to do such an analysis).

The analysis in \citep{Even-DarKM07} (which assumed that approximations to $T$ and $O$ were provided) dealt with this problem of dividing by zero (during a Bayes' rule update) by explicitly modifying the approximate model so that it \emph{never} assigns the probability of any event to be zero (since if this event occurred, then the conditional probability is no longer defined). In our setting, Condition~\ref{cond:alpha} ensures that true model never assigns the probability of any event to be zero. We can relax this condition somewhat (so that we need not quantify over all observations), though we do not discuss this here.

We should also remark that while our sample complexity bound is significantly larger than in Theorem~\ref{theorem:main-l1}, we are also bounding the more stringent KL-error measure on conditional distributions.

\subsubsection{Learning Distributions $\epsilon$-close to HMMs}

Our $L_1$ error guarantee for predicting joint probabilities still holds if
the sample used to estimate $\wh P_1, \wh P_{2,1}, \wh P_{3,x,1}$ come from
a probability distribution $\Pr[\cdot]$ that is merely close to an HMM.
Specifically, all we need is that there exists some $t_\text{max} \geq 3$ and
some $m$ state HMM with distribution $\Pr^\text{HMM}[\cdot]$ such that:
\begin{enumerate}
\item $\Pr^\text{HMM}$ satisfies Condition~\ref{cond:rank} (HMM Rank
Condition),

\item For all $t \leq t_\text{max}$,
$\sum_{x_{1:t}} |\Pr[x_{1:t}] - {\Pr}^\text{HMM}[x_{1:t}]| \leq
\epsilon^\text{HMM}(t)$,

\item $\epsilon^\text{HMM}(2) \ll \frac12 \sigma_m(P_{2,1}^\text{HMM})$.
\end{enumerate}
The resulting error of our learned model $\wh \Pr$ is
$$
\sum_{x_{1:t}} |\Pr[x_{1:t}] - \wh \Pr[x_{1:t}]|
\ \leq \
\epsilon^\text{HMM}(t) +
\sum_{x_{1:t}} |{\Pr}^\text{HMM}[x_{1:t}] - \wh \Pr[x_{1:t}]|
$$
for all $t \leq t_\text{max}$. The second term is now bounded as in
Theorem~\ref{theorem:main-l1}, with spectral parameters corresponding to
$\Pr^\text{HMM}$.


\subsection{Subsequent Work}

Following the initial publication of this work, Siddiqi, Boots, and Gordon
have proposed various extensions to the \textsc{LearnHMM} algorithm and its
analysis~\cite{SBG10}.
First, they show that the model parameterization used by our algorithm in
fact captures the class of HMMs with rank $m$ transition matrices, which is
more general than the class of HMMs with $m$ hidden states.
Second, they propose extensions for using longer sequences in the parameter
estimation, and also for handling real-valued observations.
These extensions prove to be useful in both synthetic experiments and an
application to tracking with video data.

A recent work of Song, Boots, Siddiqi, Gordon, and Smola provides a
kernelization of our model parameterization in the context of Hilbert space
embeddings of (conditional) probability distributions, and extends various
aspects of the \textsc{LearnHMM} algorithm and analysis to this
setting~\cite{SBSGS10}.
This extension is also shown to be advantageous in a number of
applications.

\section{Proofs}

Throughout this section, we assume the HMM obeys Condition~\ref{cond:rank}.
Table~\ref{table:notation} summarizes the notation that will be used throughout the analysis in this section.

\begin{table}[h]
\begin{center}
\begin{tabular}{ll}
\hline
\abovespace
$m$, $n$ & Number of states and observations \\
$n_0(\eps)$ & Number of significant observations \\
$O$, $T$, $A_x$ & HMM parameters \\
$P_1$, $P_{2,1}$, $P_{3,x,1}$ & Marginal probabilities \\
$\wh P_1$, $\wh P_{2,1}$, $\wh P_{3,x,1}$ & Empirical marginal probabilities \\
$\epsilon_1$, $\epsilon_{2,1}$, $\epsilon_{3,x,1}$ & Sampling errors
[Section~\ref{section:errors}] \\
$\wh U$ & Matrix of $m$ left singular vectors of $\wh P_{2,1}$ \\
$\wt b_\infty$, $\wt B_x$, $\wt b_1$ & True observable parameters using $\wh U$
[Section~\ref{section:errors}] \\
$\wh b_\infty$, $\wh B_x$, $\wh b_1$ & Estimated observable parameters using $\wh U$ \\
$\delta_\infty$, $\Delta_x$, $\delta_1$ & Parameter errors
[Section~\ref{section:errors}] \\
$\Delta$ & $\sum_x \Delta_x$
[Section~\ref{section:errors}] \\
$\sigma_m(M)$ & $m$-th largest singular value of matrix $M$ \\
$\vec{b}_t$, $\wh b_t$ & True and estimated states
[Section~\ref{section:tracking-proof}] \\
$\vec{h}_t$, \ $\wh h_t$, \ $\wh g_t$ & $(\wh U^\top O)^{-1} \vec{b}_t$, \ $(\wh U^\top
O)^{-1} \wh b_t$, \ $\wh h_t / (\vec{1}_m^\top \wh h_t)$
[Section~\ref{section:tracking-proof}] \\
$\wh A_x$ & $(\wh U^\top O)^{-1} \wh B_x (\wh U^\top O)$
[Section~\ref{section:tracking-proof}] \\
$\gamma$, \ $\alpha$ & $\inf \{ \|Ov\|_1 : \|v\|_1 = 1 \}$, \ $\min \{ [A_x]_{i,j} \}$ \\
\hline
\end{tabular}
\caption{Summary of notation.} \label{table:notation}
\end{center}
\vskip -0.15in
\end{table}

\subsection{Estimation Errors} \label{section:errors}

Define the following sampling error quantities:
\begin{eqnarray*}
\epsilon_1 & = & \| \wh P_1 - P_1 \|_2 \\
\epsilon_{2,1} & = & \| \wh P_{2,1} - P_{2,1} \|_2 \\
\epsilon_{3,x,1} & = & \| \wh P_{3,x,1} - P_{3,x,1} \|_2
\end{eqnarray*}
The following lemma bounds these errors with high probability as a function
of the number of observation samples used to form the estimates.
\begin{lemma} \label{lemma:sampling}
If the algorithm independently samples $N$ observation triples from the
HMM, then with probability at least $1 - \eta$:
\begin{eqnarray*}
\epsilon_1 &\leq& \sqrt{\frac1N \ln \frac3{\eta}} + \sqrt{\frac1N}
\label{eq:eps1} \\
\epsilon_{2,1} &\leq& \sqrt{\frac1N \ln \frac3{\eta}} + \sqrt{\frac1N}
\label{eq:eps21} \\
\max_x \epsilon_{3,x,1} &\leq&
\sqrt{\frac1N \ln \frac3{\eta}} + \sqrt{\frac1N}
\label{eq:maxeps3x1} \\
\sum_x \epsilon_{3,x,1} &\leq&
\min_{k} \left( \sqrt{\frac{k}{N} \ln \frac{3}{\eta}} + \sqrt{\frac{k}{N}}
 + 2 \epsilon(k) \right)
+ \sqrt{\frac1N \ln \frac3{\eta}} + \sqrt{\frac1N}
\label{eq:eps3x1}
\end{eqnarray*}
where $\epsilon(k)$ is defined in \eqref{eq:epsk}.
\end{lemma}
\begin{proof}
See Appendix~\ref{section:sampling}.
\end{proof}
The rest of the analysis estimates how the sampling errors affect the
accuracies of the model parameters (which in turn affect the prediction
quality).
We need some results from matrix perturbation theory, which are given in
Appendix~\ref{section:matrix}.

Let $U \in \R^{n \times m}$ be matrix of left singular vectors of
$P_{2,1}$. The first lemma implies that if $\wh P_{2,1}$ is sufficiently
close to $P_{2,1}$, \emph{i.e.}~$\epsilon_{2,1}$ is small enough, then the
difference between projecting to $\range(\wh U)$ and to $\range(U)$ is
small. In particular, $\wh U^\top O$ will be invertible and be nearly as
well-conditioned as $U^\top O$.
\begin{lemma} \label{lemma:subspace}
Suppose $\epsilon_{2,1} \leq \eps \cdot \sigma_m(P_{2,1})$ for some
$\eps < 1/2$.
Let $\eps_0 = \epsilon_{2,1}^2/((1-\eps)\sigma_m(P_{2,1}))^2$.
Then:
\begin{enumerate}
\item $\eps_0 < 1$,
\item $\sigma_m(\wh U^\top \wh P_{2,1}) \geq (1-\eps) \sigma_m(P_{2,1})$,
\item $\sigma_m(\wh U^\top P_{2,1}) \geq \sqrt{1-\eps_0} \sigma_m(P_{2,1})$,
\item $\sigma_m(\wh U^\top O) \geq \sqrt{1-\eps_0} \sigma_m(O)$.
\end{enumerate}
\end{lemma}
\begin{proof}
The assumptions imply $\eps_0 < 1$.
Since $\sigma_m(\wh U^\top \wh P_{2,1}) = \sigma_m(\wh P_{2,1})$, the
second claim is immediate from Corollary~\ref{cor:svd}.
Let $U \in \R^{n \times m}$ be the matrix of left singular vectors of
$P_{2,1}$.
For any $x \in \R^m$, $\|\wh U^\top U x\|_2 = \|x\|_2 \sqrt{1 - \|\wh
U_\perp^\top U\|_2^2} \geq \|x\|_2 \sqrt{1 - \eps_0}$ by
Corollary~\ref{cor:svd} and the fact $\eps_0 < 1$.
The remaining claims follow.
\end{proof}

Now we will argue that the estimated parameters $\wh b_\infty, \wh B_x, \wh
b_1$ are close to the following true parameters from the observable representation
when $\wh U$ is used for $U$:
\begin{eqnarray*}
\wt b_\infty & = & (P_{2,1}^\top \wh U)^+ P_1 \ = \  (\wh U^\top O)^{-\top}
\vec{1}_m, \\
\wt B_x & = & (\wh U^\top P_{3,x,1}) (\wh U^\top P_{2,1})^+
\ = \ (\wh U^\top O) A_x (\wh U^\top O)^{-1} \quad \text{for $x = 1,
\ldots, n$}, \\
\wt b_1 & = & \wh U^\top P_1.
\end{eqnarray*}
By Lemma~\ref{lemma:obs},
as long as $\wh U^\top O$ is invertible, these parameters $\wt b_\infty,
\wt B_x, \wt b_1$ constitute a valid observable representation for the HMM.

Define the following errors of the estimated parameters:
\begin{eqnarray*}
\delta_\infty & = &
\left\| (\wh U^\top O)^\top (\wh b_\infty - \wt b_\infty)\right\|_\infty
\ = \ \left\|(\wh U^\top O)^\top \wh b_\infty - \vec{1}_m\right\|_\infty , \\
\Delta_x & = & \left\|
(\wh U^\top O)^{-1} \left( \wh B_x - \wt B_x \right) (\wh U^\top O)
\right\|_1
\ = \ \left\| (\wh U^\top O)^{-1} \wh B_x (\wh U^\top O) - A_x \right\|_1 , \\
\Delta & = & \sum_x \Delta_x \\
\delta_1 & = & \left\|(\wh U^\top O)^{-1} (\wh b_1 - \wt b_1)\right\|_1
\ = \ \left\| (\wh U^\top O)^{-1} \wh b_1 - \vec{\pi} \right\|_1 .
\end{eqnarray*}
We can relate these to the sampling errors as follows.
\begin{lemma} \label{lemma:parameters}
Assume $\epsilon_{2,1} \leq \sigma_m(P_{2,1})/ 3$. Then:
\begin{eqnarray*}
\delta_\infty & \leq &
4 \cdot \left(
\frac{\epsilon_{2,1}}{\sigma_m(P_{2,1})^2}
+ \frac{\epsilon_1}{3\sigma_m(P_{2,1})} \right), \\
\Delta_x & \leq & 
\frac{8}{\sqrt{3}} \cdot \frac{\sqrt{m}}{\sigma_m(O)} \cdot \left(
\Pr[x_2=x] \cdot \frac{\epsilon_{2,1}}{\sigma_m(P_{2,1})^2}
+ \frac{\epsilon_{3,x,1}}{3 \sigma_m(P_{2,1})} \right) ,\\
\Delta & \leq & 
\frac{8}{\sqrt{3}} \cdot \frac{\sqrt{m}}{\sigma_m(O)} \cdot \left(
\frac{\epsilon_{2,1}}{\sigma_m(P_{2,1})^2} + \frac{\sum_x
\epsilon_{3,x,1}}{3 \sigma_m(P_{2,1})} \right) ,\\
\delta_1 & \leq & 
\frac{2}{\sqrt{3}} \cdot \frac{\sqrt{m}}{\sigma_m(O)} \cdot \epsilon_1.
\end{eqnarray*}
\end{lemma}

\begin{proof}
The assumption on $\epsilon_{2,1}$ guarantees that $\wh U^\top O$ is
invertible (Lemma~\ref{lemma:subspace}).

We bound $\delta_\infty = \|(O^\top U) (\wh b_\infty - \wt b_\infty)\|_\infty$ by
$\|O^\top\|_\infty \|U (\wh b_\infty - \wt b_\infty)\|_\infty \leq
\|\wh b_\infty - \wt b_\infty\|_2$. Then:
\begin{eqnarray*}
\|\wh b_\infty - \wt b_\infty\|_2
& = & \|(\wh P_{2,1}^\top \wh U)^+ \wh P_1 - (P_{2,1}^\top \wh U)^+ P_1\|_2
\\
& \leq & \|((\wh P_{2,1}^\top \wh U)^+ - (P_{2,1}^\top \wh U)^+ ) \wh P_1 \|_2
+ \|(P_{2,1}^\top \wh U)^+ (\wh P_1 - P_1)\|_2 \\
& \leq &
\|((\wh P_{2,1}^\top \wh U)^+ - (P_{2,1}^\top \wh U)^+ )\|_2 \|\wh P_1\|_1
+ \|(P_{2,1}^\top \wh U)^+\|_2 \|\wh P_1 - P_1\|_2 \\
& \leq &
\frac{1 + \sqrt{5}}{2} \cdot
\frac{\epsilon_{2,1}}{\min\{\sigma_m(\wh P_{2,1}),\sigma_m(P_{2,1}^\top \wh
U)\}^2}
+ \frac{\epsilon_1}{\sigma_m(P_{2,1}^\top \wh U)},
\end{eqnarray*}
where the last inequality follows from Lemma~\ref{lemma:pseudoinverse}.
The bound now follows from Lemma~\ref{lemma:subspace}.

\sloppy
Next for $\Delta_x$, we bound each term $\|(\wh U^\top O)^{-1} (\wh B_x -
\wt B_x) (\wh U^\top O)\|_1$ by $\sqrt{m} \|(\wh U^\top O)^{-1} (\wh B_x -
\wt B_x) \wh U^\top \|_2 \|O\|_1 \leq \sqrt{m} \|(\wh U^\top O)^{-1}\|_2
\|\wh B_x - \wt B_x\|_2 \|\wh U^\top\|_2 \|O\|_1 = \sqrt{m} \| \wh B_x -
\wt B_x\|_2 / \sigma_m(\wh U^\top O)$.
To deal with $\| \wh B_x - \wt B_x\|_2$, we use the decomposition
\begin{eqnarray*}
\left\| \wh B_x - \wt B_x \right\|_2
& = & \left\| (\wh U^\top P_{3,x,1})(\wh U^\top P_{2,1})^+
- (\wh U^\top \wh P_{3,x,1}) (\wh U^\top \wh P_{2,1})^+ \right\|_2 \\
& \leq & \left\| (\wh U^\top P_{3,x,1})\left( (\wh U^\top P_{2,1})^+
- (\wh U^\top \wh P_{2,1})^+ \right) \right\|_2 \\
& & + \left\| \wh U^\top \left( P_{3,x,1} - \wh P_{3,x,1} \right)
(\wh U^\top P_{2,1})^+ \right\|_2 \\
& \leq & \|P_{3,x,1}\|_2
\cdot \frac{1+\sqrt{5}}{2} \cdot
\frac{\epsilon_{2,1}}{\min\{\sigma_m(\wh P_{2,1}),\sigma_m(\wh U^\top
P_{2,1})\}^2} \\
& & + \frac{\epsilon_{3,x,1}}{\sigma_m(\wh U^\top P_{2,1})} \\
& \leq & \Pr[x_2 = x]
\cdot \frac{1+\sqrt{5}}{2} \cdot
\frac{\epsilon_{2,1}}{\min\{\sigma_m(\wh P_{2,1}),\sigma_m(\wh U^\top
P_{2,1})\}^2} \\
& & + \frac{\epsilon_{3,x,1}}{\sigma_m(\wh U^\top P_{2,1})},
\end{eqnarray*}
where the second inequality uses Lemma~\ref{lemma:pseudoinverse}, and the
final inequality uses the fact $\|P_{3,x,1}\|_2 \leq \sqrt{\sum_{i,j}
[P_{3,x,1}]_{i,j}^2 }\leq \sum_{i,j} [P_{3,x,1}]_{i,j} = \Pr[x_2 = x]$.
Applying Lemma~\ref{lemma:subspace} gives the
stated bound on $\Delta_x$ and also $\Delta$.
\fussy

Finally, we bound $\delta_1$ by $\sqrt{m} \|(\wh U^\top O)^{-1} \wh U^\top
\|_2 \|\wh P_1 - P_1\|_2 \leq \sqrt{m} \epsilon_1 / \sigma_m(\wh U^\top
O)$. Again, the stated bound follows from Lemma~\ref{lemma:subspace}.
\end{proof}

\subsection{Proof of Theorem~\ref{theorem:main-l1}}
\label{section:l1-proof}

We need to quantify how estimation errors propagate in the
probability calculation. Because the joint probability of a length $t$
sequence is computed by multiplying together $t$ matrices, there is a
danger of magnifying the estimation errors exponentially. Fortunately,
this is not the case: the following lemma shows that these errors
accumulate roughly additively.
\begin{lemma} \label{lemma:l1-induction}
Assume $\wh U^\top O$ is invertible.
For any time $t$:
$$ \sum_{x_{1:t}}
\left\| (\wh U^\top O)^{-1} \left( \wh B_{x_{t:1}} \wh b_1 \ - \ \wt
B_{x_{t:1}} \wt b_1 \right) \right\|_1
\leq (1+\Delta)^t \delta_1 + (1+\Delta)^t - 1. $$
\end{lemma}

\begin{proof}
By induction on $t$. The base case, that $\|(\wh U^\top O)^{-1} (\wh b_1 -
\wt b_1)\|_1 \leq (1+\Delta)^0 \delta_1 + (1+\Delta)^0 - 1 = \delta_1$ 
is true by definition.
For the inductive step, define unnormalized states
$\wh b_t = \wh b_t(x_{1:t-1}) = \wh B_{x_{t-1:1}} \wh b_1$
and $\wt b_t = \wt b_t(x_{1:t-1}) = \wt B_{x_{t-1:1}} \wt b_1$.
Fix $t > 1$, and assume
$$ \sum_{x_{1:t-1}}
\left\| (\wh U^\top O)^{-1} \left( \wh b_{t} \ - \ \wt b_{t} \right) \right\|_1
\leq (1+\Delta)^{t-1} \delta_1 + (1+\Delta)^{t-1} - 1. $$
Then, we can decompose the sum over $x_{1:t}$ as
\begin{align*}
\lefteqn{\sum_{x_{1:t}}
\|(\wh U^\top O)^{-1} (\wh B_{x_{t:1}} \wh b_1
- \wt B_{x_{t:1}} \wt b_1)\|_1} \\
& = \sum_{x_{1:t}}
\left\|(\wh U^\top O)^{-1} \left(
\left( \wh B_{x_t} - \wt B_{x_t} \right) \wt b_t
  +   \left( \wh B_{x_t} - \wt B_{x_t} \right)
\left( \wh b_t - \wt b_t \right)
  +   \wt B_{x_t} \left( \wh b_t - \wt b_t \right)
\right)\right\|_1,
\end{align*}
which, by the triangle inequality, is bounded above by
\begin{eqnarray}
& & \sum_{x_t} \sum_{x_{1:t-1}} \left\|
(\wh U^\top O)^{-1} \left( \wh B_{x_t} - \wt B_{x_t} \right) (\wh U^\top O)
\right\|_1 \left\| (\wh U^\top O)^{-1} \wt b_t \right\|_1
\label{eq:ind-sum1} \\
& & + \ \sum_{x_t} \sum_{x_{1:t-1}} \left\|
(\wh U^\top O)^{-1} \left( \wh B_{x_t} - \wt B_{x_t} \right) (\wh U^\top O)
\right\|_1 \left\|
(\wh U^\top O)^{-1} \left( \wh b_t - \wt b_t \right) \right\|_1
\label{eq:ind-sum2} \\
& & + \ \sum_{x_t} \sum_{x_{1:t-1}} \left\|
(\wh U^\top O)^{-1} \wt B_t (\wh U^\top O)
(\wh U^\top O)^{-1} \left( \wh b_t - \wt b_t \right) \right\|_1.
\label{eq:ind-sum3}
\end{eqnarray}
We deal with each double sum individually.
For the sums in \eqref{eq:ind-sum1}, we use the fact that $\|(\wh U^\top
O)^{-1} \wt b_t\|_1 = \Pr[x_{1:t-1}]$, which, when summed over $x_{1:t-1}$,
is $1$. Thus the entire double sum is bounded by $\Delta$ by definition.
For \eqref{eq:ind-sum2}, we use the inductive hypothesis to bound the inner
sum over $\|(\wh U^\top O) (\wh b_t - \wt b_t)\|_1$; the outer sum scales
this bound by $\Delta$ (again, by definition). Thus the double sum is
bounded by $\Delta ( (1+\Delta)^{t-1}\delta_1 + (1+\Delta)^{t-1} - 1)$.
Finally, for sums in \eqref{eq:ind-sum3}, we first replace $(\wh U^\top
O)^{-1} \wt B_t (\wh U^\top O)$ with $A_{x_t}$. Since $A_{x_t}$ has all
non-negative entries, we have that $\|A_{x_t}\vec{v}\|_1 \leq \vec{1}_m^\top A_{x_t}
|\vec{v}|$ for any vector $\vec{v} \in \R^m$, where $|\vec{v}|$ denotes element-wise absolute
value of $\vec{v}$. Now the fact $\vec{1}_m^\top \sum_{x_t} A_{x_t}
|\vec{v}| = \vec{1}_m^\top T
|\vec{v}| = \vec{1}_m^\top |\vec{v}| = \|\vec{v}\|_1$ and the inductive hypothesis imply the double
sum in \eqref{eq:ind-sum3} is bounded by $(1 + \Delta)^{t-1} \delta_1 + (1 +
\Delta)^{t-1} - 1$.
Combining these bounds for \eqref{eq:ind-sum1}, \eqref{eq:ind-sum2}, and
\eqref{eq:ind-sum3} completes the induction.
\end{proof}

All that remains is to bound the effect of errors in $\wh b_\infty$.
Theorem~\ref{theorem:main-l1} will follow from the following lemma combined
with the sampling error bounds of Lemma~\ref{lemma:sampling}.
\begin{lemma} \label{lemma:l1}
Assume $\epsilon_{2,1} \leq \sigma_m(P_{2,1})/3$.
Then for any $t$,
$$
\sum_{x_{1:t}}
\left| \Pr[x_{1:t}] \ - \ \wh \Pr[x_{1:t}] \right|
\ \leq \
\delta_\infty \ + \ (1 + \delta_\infty) \left( (1 + \Delta)^t \delta_1 \ + \ (1 +
\Delta)^t \ - 1 \right).
$$
\end{lemma}
\begin{proof}
By Lemma~\ref{lemma:subspace} and the condition on $\epsilon_{2,1}$,
we have $\sigma_m(\wh U^\top O) > 0$ so $\wh U^\top O$ is invertible.

Now we can decompose the $L_1$ error as follows:
\begin{eqnarray}
\lefteqn{\sum_{x_{1:t}}
\left| \wh \Pr[x_{1:t}] \ - \ \Pr[x_{1:t}] \right|
\ = \ \sum_{x_{1:t}} \left|
\wh b_\infty^\top \wh B_{x_{t:1}} \wh b_1 \ - \ \vec{b}_\infty^\top
B_{x_{t:1}} \vec{b}_1
\right|} \nonumber \\
& = & \sum_{x_{1:t}} \left| \wh b_\infty^\top \wh B_{x_{t:1}} \wh b_1 \ - \ \wt
b_\infty^\top \wt B_{x_{t:1}} \wt b_1 \right| \nonumber \\
& \leq & \sum_{x_{1:t}}
\left| (\wh b_\infty - \wt b_\infty)^\top (\wh U^\top O) (\wh U^\top O)^{-1}
\wt B_{x_{t:1}} \wt b_1 \right|
\label{eq:l1-sum1} \\
& & + \ \sum_{x_{1:t}}
\left| (\wh b_\infty - \wt b_\infty)^\top (\wh U^\top O) (\wh U^\top O)^{-1}
( \wh B_{x_{t:1}} \wh b_1 - \wt B_{x_{t:1}} \wt b_1) \right|
\label{eq:l1-sum2} \\
& & + \ \sum_{x_{1:t}}
\left| \wt b_\infty^\top (\wh U^\top O) (\wh U^\top O)^{-1}
( \wh B_{x_{t:1}} \wh b_1 - \wt B_{x_{t:1}} \wt b_1) \right|.
\label{eq:l1-sum3}
\end{eqnarray}
The first sum \eqref{eq:l1-sum1} is
\begin{eqnarray*}
\lefteqn{\sum_{x_{1:t}}
\left| (\wh b_\infty - \wt b_\infty)^\top (\wh U^\top O) (\wh U^\top O)^{-1}
\wt B_{x_{t:1}} \wt b_1 \right|
}\\
& \leq &
\sum_{x_{1:t}}
\left\| (\wh U^\top O)^\top (\wh b_\infty - \wt b_\infty) \right\|_\infty \left\|
(\wh U^\top O)^{-1} \wt B_{x_{t:1}} \wt b_1 \right\|_1 \\
& \leq &
\sum_{x_{1:t}}
\delta_\infty \ \| A_{x_{t:1}} \vec{\pi} \|_1 
\ = \
\sum_{x_{1:t}} \delta_\infty \Pr[x_{1:t}]
\ = \ \delta_\infty
\end{eqnarray*}
where the first inequality is H\"older's, and the second uses the
bounds in Lemma~\ref{lemma:parameters}.

The second sum \eqref{eq:l1-sum2} employs H\"older's and
Lemma~\ref{lemma:l1-induction}:
\begin{eqnarray*}
\lefteqn{\left| (\wh b_\infty - \wt b_\infty)^\top (\wh U^\top O) (\wh U^\top O)^{-1}
( \wh B_{x_{t:1}} \wh b_1 - \wt B_{x_{t:1}} \wt b_1) \right|
}
\\
& \leq &
\left\| (\wh U^\top O)^\top (\wh b_\infty - \wt b_\infty) \right\|_\infty
\left\| (\wh U^\top O)^{-1}
( \wh B_{x_{t:1}} \wh b_1 - \wt B_{x_{t:1}} \wt b_1) \right\|_1 \\
& \leq &
\delta_\infty ( (1 + \Delta)^t \delta_1 + (1 + \Delta)^t - 1 ).
\end{eqnarray*}

Finally, the third sum \eqref{eq:l1-sum3} uses
Lemma~\ref{lemma:l1-induction}:
\begin{eqnarray*}
\lefteqn{\sum_{x_{1:t}}
\left| \wt b_\infty^\top (\wh U^\top O) (\wh U^\top O)^{-1}
( \wh B_{x_{t:1}} \wh b_1 - \wt B_{x_{t:1}} \wt b_1) \right|
}
\\
& = &
\sum_{x_{1:t}}
\left| 1^\top (\wh U^\top O)^{-1} ( \wh B_{x_{t:1}} \wh b_1 - \wt
B_{x_{t:1}} \wt b_1) \right| \\
& \leq & 
\sum_{x_{1:t}}
\left\| (\wh U^\top O)^{-1} ( \wh B_{x_{t:1}} \wh b_1 - \wt
B_{x_{t:1}} \wt b_1) \right\|_1 \\
& \leq &
(1 + \Delta)^t \delta_1 + (1 + \Delta)^t - 1.
\end{eqnarray*}
Combining these gives the desired bound.
\end{proof}

\begin{proof}[Proof of Theorem~\ref{theorem:main-l1}]
By Lemma~\ref{lemma:sampling}, the specified number of samples $N$ (with
a suitable constant $C$), together with the setting of $\eps$ in
$n_0(\eps)$, guarantees the following sampling error bounds:
\begin{align*}
\epsilon_1 & \leq \min\left( 0.05 \cdot (3/8) \cdot \sigma_m(P_{2,1}) \cdot
\epsilon,
\ 0.05 \cdot (\sqrt3/2) \cdot \sigma_m(O) \cdot (1/\sqrt{m}) \cdot \epsilon
\right) \\
\epsilon_{2,1} & \leq \min\left( 0.05 \cdot (1/8) \cdot \sigma_m(P_{2,1})^2
\cdot (\epsilon/5), \right. \\
& \hphantom{\leq \min\ } \left.
\ 0.01 \cdot (\sqrt3/8) \cdot \sigma_m(O) \cdot \sigma_m(P_{2,1})^2 \cdot
(1 / (t\sqrt{m})) \cdot \epsilon \right) \\
\sum_x \epsilon_{3,x,1} & \leq 0.39 \cdot (3\sqrt3/8) \cdot \sigma_m(O)
\cdot \sigma_m(P_{2,1}) \cdot (1/(t\sqrt{m})) \cdot \epsilon
.
\end{align*}
These, in turn, imply the following parameter error bounds, via
Lemma~\ref{lemma:parameters}:
$\delta_\infty \leq 0.05\epsilon$, $\delta_1 \leq 0.05 \epsilon$, and
$\Delta \leq 0.4 \epsilon/t$.
Finally, Lemma~\ref{lemma:l1} and the fact $(1+a/t)^t \leq 1+2a$ for $a
\leq 1/2$, imply the desired $L_1$ error bound of $\epsilon$.
\end{proof}

\subsection{Proof of Theorem~\ref{theorem:main-tracking}}
\label{section:tracking-proof}

In this subsection, we assume the HMM obeys Condition~\ref{cond:alpha} (in
addition to Condition~\ref{cond:rank}).

We introduce the following notation. Let the unnormalized estimated
conditional hidden state distributions be
$$ \wh h_t = (\wh U^\top O)^{-1} \wh b_t, $$
and its normalized version,
$$ \wh g_t \ = \ \wh h_t/(\vec{1}_m^\top \wh h_t).
$$
Also, let
$$ \wh A_x \ = \ (\wh U^\top O)^{-1} \wh B_x (\wh U^\top O).$$
This notation lets us succinctly compare the updates made by our estimated
model to the updates of the true model.
Our algorithm never explicitly computes these hidden state distributions
$\wh g_t$ (as it would require knowledge of the unobserved $O$). However,
under certain conditions (namely Conditions~\ref{cond:rank} and
\ref{cond:alpha} and some estimation accuracy requirements), these
distributions are well-defined and thus we use them for sake of analysis.

The following lemma shows that if the estimated parameters are accurate,
then the state updates behave much like the true hidden state updates.
\begin{lemma} \label{lemma:approx-update}
For any probability vector $\vec{w} \in \R^m$ and any observation $x$,
\begin{eqnarray*}
\left| \sum_x \wh b_\infty^\top (\wh U^\top O) \wh A_x \vec{w} - 1 \right|
& \leq & \delta_\infty + \delta_\infty \Delta + \Delta \quad \text{and} \\
\frac{[\wh A_x \vec{w}]_i}{\wh b_\infty^\top (\wh U^\top O) \wh A_x \vec{w}}
& \geq &
\frac{[A_x\vec{w}]_i - \Delta_x}{\vec{1}_m^\top A_x \vec{w} + \delta_\infty + \delta_\infty
\Delta_x +
\Delta_x} \quad \text{for all $i = 1, \ldots, m$}
\end{eqnarray*}
Moreover, for any non-zero vector $\vec{w} \in \R^m$,
$$ \frac{\vec{1}_m^\top \wh A_x \vec{w}}{\wh b_\infty^\top (\wh U^\top O)
\wh A_x \vec{w}}
\ \leq \
\frac{1}{1-\delta_\infty}. $$
\end{lemma}
\begin{proof}
We need to relate the effect of the estimated
operator $\wh A_x$ to that of the true operator $A_x$. First assume
$\vec{w}$ is a probability vector. Then:
\begin{eqnarray*}
\lefteqn{\left| \wh b_\infty^\top (\wh U^\top O) \wh A_x \vec{w} -
\vec{1}_m^\top A_x \vec{w} \right|}
\\
& = & \left| (\wh b_\infty - \wt b_\infty)^\top (\wh U^\top O) A_x \vec{w}
\right.  \\
& & \left.
\ + \ (\wh b_\infty - \wt b_\infty)^\top (\wh U^\top O) (\wh A_x - A_x)
\vec{w}
\ + \ \wt b_\infty (\wh U^\top O) (\wh A_{x} - A_{x}) \vec{w} \right| \\
& \leq &
\|(\wh b_\infty - \wt b_\infty)^\top (\wh U^\top O)\|_\infty \|A_x \vec{w}\|_1 \\
& &
\ + \
\|(\wh b_\infty - \wt b_\infty)^\top (\wh U^\top O)\|_\infty 
\|(\wh A_x - A_x)\|_1 \|\vec{w}\|_1
\ + \
\|(\wh A_x - A_x)\|_1 \|\vec{w}\|_1.
\end{eqnarray*}
Therefore we have
\[
\left| \sum_x \wh b_\infty^\top (\wh U^\top O) \wh A_x \vec{w}
- 1 \right|
\ \leq \ \delta_\infty + \delta_\infty \Delta + \Delta
\]
and
\[
\wh b_\infty^\top (\wh U^\top O) \wh A_x \vec{w} \ \leq \ \vec{1}_m^\top
A_x \vec{w} +
\delta_\infty + \delta_\infty \Delta_x + \Delta_x.
\]
Combining these inequalities with
\begin{eqnarray*}
[\wh A_x \vec{w}]_i
& = & [A_x \vec{w}]_i + [(\wh A_x - A_x)\vec{w}]_i \\
& \geq & [A_x \vec{w}]_i - \|(\wh A_x - A_x)\vec{w}\|_1 \\
& \geq & [A_x \vec{w}]_i - \|(\wh A_x - A_x)\|_1 \|\vec{w}\|_1 \\
& \geq & [A_x \vec{w}]_i - \Delta_x
\end{eqnarray*}
gives the first claim.

Now drop the assumption that $\vec{w}$ is a probability vector, and assume
$\vec{1}_m^\top \wh A_x \vec{w} \neq 0$ without loss of generality. Then:
\begin{eqnarray*}
\frac
{\vec{1}_m^\top \wh A_x \vec{w}}
{\wh b_\infty^\top (\wh U^\top O) \wh A_x \vec{w}}
& = & \frac
{\vec{1}_m^\top \wh A_x \vec{w}}
{\vec{1}_m^\top \wh A_x \vec{w}
+ (\wh b_\infty - \wt b_\infty)^\top (\wh U^\top O) \wh A_x \vec{w}} \\
& \leq & \frac
{\|\wh A_x \vec{w}\|_1}
{\|\wh A_x \vec{w}\|_1 - 
\|(\wh U^\top O)^\top (\wh b_\infty - \wt b_\infty)\|_\infty
\|\wh A_x \vec{w}\|_1}
\end{eqnarray*}
which is at most $1/(1-\delta_\infty)$ as claimed.
\end{proof}

A consequence of Lemma~\ref{lemma:approx-update} is that if the estimated
parameters are sufficiently accurate, then the state
updates never allow predictions of very small hidden state probabilities.
\begin{corollary} \label{corollary:belief-lb}
Assume
$\delta_\infty \leq 1/2$,
$\max_x \Delta_x \leq \alpha/3$,
$\delta_1 \leq \alpha/8$,
and $\max_x \delta_\infty + \delta_\infty\Delta_x + \Delta_x \leq 1/3$.
Then $[\wh g_t]_i \geq \alpha/2$ for all $t$ and $i$.
\end{corollary}
\begin{proof}
For $t = 1$,
we use Lemma~\ref{lemma:parameters} to get
$\|\vec{h}_1 - \wh h_1\|_1 \leq \delta_1 \leq 1/2$, so
Lemma~\ref{lemma:normalize} implies that $\|\vec{h}_1 - \wh g_1\|_1 \leq
4\delta_1$.
Then $[\wh g_1]_i \geq [\vec{h}_1]_i - |[\vec{h}_1]_i - [\wh g_1]_i| \geq
\alpha - 4\delta_1 \geq \alpha/2$ (using Condition~\ref{cond:alpha})
as needed.
For $t > 1$, Lemma~\ref{lemma:approx-update} implies
$$ \frac{[\wh A_x \wh g_{t-1}]_i}
{\vec{b}_\infty^\top (\wh U^\top O) \wh A_x \wh g_{t-1}}
\ \geq \
\frac{[A_x \wh g_{t-1}]_i - \Delta_x}
{\vec{1}_m^\top A_x \wh g_{t-1} + \delta_\infty + \delta_\infty \Delta_x + \Delta_x}
\ \geq \
\frac{\alpha - \alpha/3}{1 + 1/3}
\ \geq \
\frac{\alpha}{2}
$$
using Condition~\ref{cond:alpha} in the second-to-last step.
\end{proof}

Lemma~\ref{lemma:approx-update} and Corollary~\ref{corollary:belief-lb}
can now be used to prove the contraction property of the KL-divergence
between the true hidden states and the estimated hidden states.
The analysis shares ideas from \cite{Even-DarKM07}, though the added
difficulty is due to the fact that the state maintained by our algorithm is
not a probability distribution.
\begin{lemma} \label{lemma:kl-recurrence}
Let $\eps_0 = \max_x 2\Delta_x/\alpha + (\delta_\infty +
\delta_\infty\Delta_x + \Delta_x)/\alpha + 2\delta_\infty$.
Assume $\delta_\infty \leq 1/2$, $\max_x \Delta_x \leq \alpha/3$, and
$\max_x \delta_\infty + \delta_\infty\Delta_x + \Delta_x \leq 1/3$.
For all $t$, if $\wh g_t \in \R^m$ is a probability vector, then
$$ KL(\vec{h}_{t+1} || \wh g_{t+1})
\ \leq \ KL(\vec{h}_t || \wh g_t) -
\frac{\gamma^2}{2\left(\ln \frac{2}{\alpha} \right)^2}
KL(\vec{h}_t || \wh g_t)^2 + \eps_0. $$
\end{lemma}
\begin{proof}
The LHS, written as an expectation over $x_{1:t}$, is
$$
KL(\vec{h}_{t+1} || \wh g_{t+1})
\ = \
\E_{x_{1:t}}
\left[
\sum_{i=1}^m
[\vec{h}_{t+1}]_i \ln \frac{[\vec{h}_{t+1}]_i}{[\wh g_{t+1}]_i}
\right].
$$
We can bound $\ln(1/[\wh g_{t+1}]_i)$ as
\begin{eqnarray*}
\ln \frac{1}{[\wh g_{t+1}]_i}
& = & \ln \left(
\frac
{\wh b_\infty^\top (\wh U^\top O) \wh A_{x_t} \wh g_t}
{[\wh A_{x_t} \wh g_t]_i}
\cdot {\vec{1}_m^\top \wh h_{t+1}}
\right) \\
& = &
\ln \left(
\frac {\vec{1}_m^\top A_{x_t} \wh g_t} {[A_{x_t} \wh g_t]_i}
\cdot \frac {[A_{x_t} \wh g_t]_i} {[\wh A_{x_t} \wh g_t]_i}
\cdot \frac {\wh b_\infty^\top (\wh U^\top O) \wh A_{x_t} \wh g_t}
{\vec{1}_m^\top A_{x_t} \wh g_t}
\cdot \vec{1}_m^\top \wh h_{t+1}
\right) \\
& \leq &
\ln \left(
\frac {\vec{1}_m^\top A_{x_t} \wh g_t} {[A_{x_t} \wh g_t]_i}
\ \cdot \ \frac {[A_{x_t} \wh g_t]_i} {[A_{x_t} \wh g_t]_i - \Delta_{x_t}}
\right. \\
& &
\hphantom{\ln\ }
\left.
\cdot \ \frac
{\vec{1}_m^\top A_{x_t} \wh g_t + \delta_\infty + \delta_\infty\Delta_{x_t} + \Delta_{x_t}}
{\vec{1}_m^\top A_{x_t} \wh g_t}
\ \cdot \ (1 + 2\delta_\infty)
\right) \\
& \leq &
\ln \left(
\frac {\vec{1}_m^\top A_{x_t} \wh g_t} {[A_{x_t} \wh g_t]_i}
\right) 
\ + \ \frac{2\Delta_{x_t}}{\alpha}
\ + \ \frac {\delta_\infty + \delta_\infty \Delta_{x_t} + \Delta_{x_t}} {\alpha}
\ + \ 2\delta_\infty \\
& \leq &
\ln \left(
\frac {\vec{1}_m^\top A_{x_t} \wh g_t} {[A_{x_t} \wh g_t]_i}
\right) 
\ + \ \eps_0
\end{eqnarray*}
where the first inequality follows from Lemma~\ref{lemma:approx-update},
and the second uses $\ln (1 + a) \leq a$.
Therefore,
\begin{equation}
KL(\vec{h}_{t+1} || \wh g_{t+1})
\ \leq \
\E_{x_{1:t}} \left[
\sum_{i=1}^m
[\vec{h}_{t+1}]_i
\ln \left(
[\vec{h}_{t+1}]_i
\cdot \frac {\vec{1}_m^\top A_{x_t} \wh g_t} {[A_{x_t} \wh g_t]_i}
\right) \right] \ + \ \eps_0.
\label{eq:kl1}
\end{equation}
The expectation in \eqref{eq:kl1} is the KL-divergence between
$\Pr[h_t|x_{1:t-1}]$ and the distribution over $h_{t+1}$ that is arrived at
by updating $\wh \Pr[h_t|x_{1:t-1}]$ (using Bayes' rule) with
$\Pr[h_{t+1}|h_t]$ and $\Pr[x_t|h_t]$.
Call this second distribution $\wt \Pr[h_{t+1}|x_{1:t}]$.
The chain rule for KL-divergence states
\begin{eqnarray*}
\lefteqn{KL(\Pr[h_{t+1}|x_{1:t}]||\wt \Pr[h_{t+1}|x_{1:t}]) \ + \
KL(\Pr[h_t|h_{t+1},x_{1:t}]||\wt \Pr[h_t|h_{t+1},x_{1:t}])} \\
& = &
KL(\Pr[h_t|x_{1:t}]||\wt \Pr[h_t|x_{1:t}]) \ + \
KL(\Pr[h_{t+1}|h_t,x_{1:t}]||\wt \Pr[h_{t+1}|h_t,x_{1:t}]).
\end{eqnarray*}
Thus, using the non-negativity of KL-divergence, we have
\begin{eqnarray*}
\lefteqn{KL(\Pr[h_{t+1}|x_{1:t}]||\wt \Pr[h_{t+1}|x_{1:t}])}
\\
& \leq &
KL(\Pr[h_t|x_{1:t}]||\wt \Pr[h_t|x_{1:t}]) +
KL(\Pr[h_{t+1}|h_t,x_{1:t}]||\wt \Pr[h_{t+1}|h_t,x_{1:t}]) \\
& = &
KL(\Pr[h_t|x_{1:t}]||\wt \Pr[h_t|x_{1:t}])
\end{eqnarray*}
where the equality follows from the fact that
$\wt \Pr[h_{t+1}|h_t,x_{1:t}]
= \wt \Pr[h_{t+1}|h_t]
= \Pr[h_{t+1}|h_t]
= \Pr[h_{t+1}|h_t,x_{1:t}]$.
Furthermore,
\[
\Pr[h_t = i|x_{1:t}]
=
\Pr[h_t = i|x_{1:t-1}] \cdot \frac{\Pr[x_t|h_t=i]}
{\sum_{j=1}^m \Pr[x_t|h_t=j] \cdot \Pr[h_t=j|x_{1:t-1}]}
\]
and
\[
\wt \Pr[h_t = i|x_{1:t}]
=
\wh \Pr[h_t = i|x_{1:t-1}] \cdot \frac{\Pr[x_t|h_t=i]}
{\sum_{j=1}^m \Pr[x_t|h_t=j] \cdot \wh \Pr[h_t=j|x_{1:t-1}]},
\]
so
\begin{eqnarray*}
\lefteqn{KL(\Pr[h_t|x_{1:t}]||\wt \Pr[h_t|x_{1:t}])}
\\
& = &
\E_{x_{1:t}} \left[
\sum_{i=1}^m
\Pr[h_t = i|x_{1:t}]
\ln \frac
{\Pr[h_t = i | x_{1:t-1}]}
{\wh \Pr[h_t = i | x_{1:t-1}]}
\right] \\
& &
- \ \E_{x_{1:t}} \left[
\sum_{i=1}^m
\Pr[h_t = i|x_{1:t}]
\ln \frac
{\sum_{j=1}^m \Pr[x_t|h_t=j] \cdot \Pr[h_t=j|x_{1:t-1}]}
{\sum_{j=1}^m \Pr[x_t|h_t=j] \cdot \wh \Pr[h_t=j|x_{1:t-1}]}
\right].
\end{eqnarray*}
The first expectation is
\begin{eqnarray*}
\lefteqn{\E_{x_{1:t}} \left[
\sum_{i=1}^m
\Pr[h_t = i|x_{1:t}]
\ln \frac
{\Pr[h_t = i | x_{1:t-1}]}
{\wh \Pr[h_t = i | x_{1:t-1}]}
\right]} \\
& = & \E_{x_{1:t-1}} \left[
\sum_{x_t}
\Pr[x_t|x_{1:t-1}]
\sum_{i=1}^m
\Pr[h_t = i|x_{1:t}]
\ln \frac
{\Pr[h_t = i | x_{1:t-1}]}
{\wh \Pr[h_t = i | x_{1:t-1}]}
\right] \\
& = & \E_{x_{1:t-1}} \left[
\sum_{x_t}
\sum_{i=1}^m
\Pr[x_t|h_t = i] \cdot
\Pr[h_t = i|x_{1:t-1}]
\ln \frac
{\Pr[h_t = i | x_{1:t-1}]}
{\wh \Pr[h_t = i | x_{1:t-1}]}
\right] \\
& = & \E_{x_{1:t-1}} \left[
\sum_{x_t}
\sum_{i=1}^m
\Pr[x_t,h_t = i|x_{1:t-1}]
\ln \frac
{\Pr[h_t = i | x_{1:t-1}]}
{\wh \Pr[h_t = i | x_{1:t-1}]}
\right] \\
& = & KL(\vec{h}_t||\wh g_t),
\end{eqnarray*}
and the second expectation is
\begin{eqnarray*}
\lefteqn{\E_{x_{1:t}} \left[
\sum_{i=1}^m
\Pr[h_t = i|x_{1:t}]
\ln \frac
{\sum_{j=1}^m \Pr[x_t|h_t=j] \cdot \Pr[h_t=j|x_{1:t-1}]}
{\sum_{j=1}^m \Pr[x_t|h_t=j] \cdot \wh \Pr[h_t=j|x_{1:t-1}]}
\right]} \\
& = &
\E_{x_{1:t-1}} \left[
\sum_{x_t}
\Pr[x_t | x_{1:t-1}]
\ln \frac
{\sum_{j=1}^m \Pr[x_t|h_t=j] \cdot \Pr[h_t=j|x_{1:t-1}]}
{\sum_{j=1}^m \Pr[x_t|h_t=j] \cdot \wh \Pr[h_t=j|x_{1:t-1}]}
\right] \\
& = & KL(O\vec{h}_t||O \wh g_t).
\end{eqnarray*}
Substituting these back into \eqref{eq:kl1}, we have
$$ KL(\vec{h}_{t+1}||\wh g_{t+1})
\ \leq \
KL(\vec{h}_t || \wh g_t) - KL(O\vec{h}_t || O\wh g_t) + \eps_0. $$
It remains to bound $KL(O\vec{h}_t || O\wh g_t)$ from above.
We use Pinsker's inequality \citep{cover-thomas},
which states that for any distributions $\vec{p}$ and $\vec{q}$,
$$
KL(\vec{p}||\vec{q}) \ \geq \ \frac12 \|\vec{p} - \vec{q}\|_1^2,
$$
together with the definition of $\gamma$, to deduce
$$ KL(O\vec{h}_t || O \wh g_t )
\ \geq \ \frac12 \E_{x_{1:t-1}} \|O\vec{h}_t - O \wh g_t\|_1^2
\ \geq \ \frac{\gamma^2}{2} \E_{x_{1:t-1}} \|\vec{h}_t - \wh g_t\|_1^2.
$$
Finally, by Jensen's inequality and Lemma~\ref{lemma:kl-lb} (the latter
applies because of Corollary~\ref{corollary:belief-lb}), we have that
$$
\E_{x_{1:t-1}} \|\vec{h}_t - \wh g_t\|_1^2
\ \geq \
(\E_{x_{1:t-1}} \|\vec{h}_t - \wh g_t\|_1)^2
\ \geq \
\left(\frac{1}{\ln \frac{2}{\alpha}}
KL(\vec{h}_t || \wh g_t) \right)^2
$$
which gives the required bound.
\end{proof}

Finally, the recurrence from Lemma~\ref{lemma:kl-recurrence} easily gives
the following lemma, which in turn combines with the sampling error bounds
of Lemma~\ref{lemma:sampling} to give Theorem~\ref{theorem:main-tracking}.
\begin{lemma} \label{lemma:tracking}
Let $\eps_0 = \max_x 2\Delta_x/\alpha + (\delta_\infty +
\delta_\infty\Delta_x + \Delta_x)/\alpha + 2\delta_\infty$
and $\eps_1 = \max_x (\delta_\infty + \sqrt{m} \delta_\infty \Delta_x + \sqrt{m}
\Delta_x)/\alpha$.
Assume $\delta_\infty \leq 1/2$, $\max_x \Delta_x \leq \alpha/3$,
$\max_x \delta_\infty + \delta_\infty\Delta_x + \Delta_x \leq 1/3$,
$\delta_1 \leq \ln(2/\alpha)/(8\gamma^2)$, 
$\eps_0 \leq \ln(2/\alpha)^2/(4\gamma^2)$, 
and $\eps_1 \leq 1/2$.
Then for all $t$,
\begin{eqnarray*}
& &
KL(\vec{h}_t || \wh g_t)
\ \leq \
\max\left( 4\delta_1 \log(2/\alpha),
\ \sqrt{\frac{2 \left( \ln \frac{2}{\alpha} \right)^2 \eps_0}{\gamma^2}}
\right)
\quad \text{and} \\
& & KL(\Pr[x_t|x_{1:t-1}] \ || \ \wh \Pr[x_t|x_{1:t-1}])
\ \leq \
KL(\vec{h}_t || \wh g_t)
\ + \ \delta_\infty + \delta_\infty\Delta + \Delta
\ + \ 2\eps_1.
\end{eqnarray*}
\end{lemma}
\begin{proof}
To prove the bound on
$KL(\vec{h}_t || \wh g_t)$,
we proceed by induction on $t$.
For the base case, Lemmas~\ref{lemma:kl-lb} (with
Corollary~\ref{corollary:belief-lb}) and \ref{lemma:normalize}
imply
$KL(\vec{h}_1 || \wh g_1)
\leq \|\vec{h}_1 - \wh g_1\|_1 \ln (2/\alpha)
\leq 4\delta_1 \ln (2/\alpha)$
as required. The inductive step follows easily from
Lemma~\ref{lemma:kl-recurrence} and simple calculus:
assuming $c_2 \leq 1/(4c_1)$, $z - c_1z^2 + c_2$ is non-decreasing in $z$
for all $z \leq \sqrt{c_2/c_1}$, so $z' \leq z - c_1z^2 + c_2$ and $z \leq
\sqrt{c_2/c_1}$ together imply that $z' \leq \sqrt{c_2/c_1}$.
The inductive step uses the the above fact with $z = KL(\vec{h}_t||\wh
g_t)$, $z' = KL(\vec h_{t+1} || \wh g_{t+1})$, $c_1 =
\gamma^2/(2(\ln(2/\alpha))^2)$, and $c_2 = \max(\eps_0, c_1 (4\delta_1
\log(2/\alpha))^2)$.

Now we prove the bound on
$KL(\Pr[x_t|x_{1:t-1}] || \wh \Pr[x_t|x_{1:t-1}])$.
First, let $\wh \Pr[x_t,h_t | x_{1:t-1}]$ denote our predicted
conditional probability of both the hidden state and observation,
\emph{i.e.}~the product of the following two quantities:
$$
\wh \Pr[h_t = i | x_{1:t-1}]
\ = \
[\wh g_t]_i
\quad \text{and} \quad
\wh \Pr[x_t | h_t = i, x_{1:t-1}]
\ = \
\frac
{[\wh b_\infty^\top (\wh U^\top O) \wh A_{x_t}]_i}
{\sum_x \wh b_\infty^\top (\wh U^\top O) \wh A_x \wh g_t}.
$$
Now we can apply the chain rule for KL-divergence
\begin{eqnarray*}
\lefteqn{KL(\Pr[x_t|x_{1:t-1}] || \wh \Pr[x_t|x_{1:t-1}])} \\
& \leq &
KL(\Pr[h_t|x_{1:t-1}] || \wh \Pr[h_t|x_{1:t-1}])
+ KL(\Pr[x_t | h_t,x_{1:t-1}] || \wh\Pr[x_t | h_t,x_{1:t-1}])
\\
& = & KL(\vec{h}_t || \wh g_t) + \E_{x_{1:t-1}}
\left[ \sum_{i=1}^m \sum_{x_t} [\vec{h}_t]_i O_{x_t,i} \ln \left(
O_{x_t,i} \cdot
\frac{\sum_x \wh b_\infty^\top (\wh U^\top O) \wh A_x \wh g_t}
{[\wh b_\infty^\top (\wh U^\top O) \wh A_{x_t}]_i}
\right) \right] \\
& \leq & KL(\vec{h}_t || \wh g_t) + \E_{x_{1:t-1}}
\left[ \sum_{i=1}^m \sum_{x_t} [\vec{h}_t]_i O_{x_t,i} \ln \left(
\frac{O_{x_t,i}}
{[\wh b_\infty^\top (\wh U^\top O) \wh A_{x_t}]_i}
\right) \right] \\
& & + \ \ln(1 + \delta_\infty + \delta_\infty\Delta + \Delta)
\end{eqnarray*}
where the last inequality uses Lemma~\ref{lemma:approx-update}.
It will suffice to show that
\[ \frac{O_{x_t,i}}{[\wh b_\infty^\top (\wh U^\top O) \wh A_{x_t}]_i}
\leq 1 + 2\eps_1
. \]
Note that $O_{x_t,i} = [\wt b_\infty^\top (\wh U^\top O) A_{x_t}]_i >
\alpha$ by Condition~\ref{cond:alpha}. Furthermore, for any $i$,
\begin{eqnarray*}
| [\wh b_\infty^\top (\wh U^\top O) \wh A_{x_t}]_i
- O_{x_t,i} |
& \leq &
\|
\wh b_\infty^\top (\wh U^\top O) \wh A_{x_t}
- \wt b_\infty^\top (\wh U^\top O) A_{x_t}
\|_\infty \\
& \leq &
\|(\wh b_\infty - \wt b_\infty) (\wh U^\top O) \|_\infty \|A_{x_t}\|_\infty \\
& & + \ \|(\wh b_\infty - \wt b_\infty) (\wh U^\top O) \|_\infty \|\wh
A_{x_t} - A_{x_t}\|_\infty \\
& & + \ \|\wt b_\infty (\wh U^\top O) \|_\infty \|\wh A_{x_t} -
A_{x_t}\|_\infty \\
& \leq & \delta_\infty + \sqrt{m}\delta_\infty\Delta_{x_t} +
\sqrt{m}\Delta_{x_t}.
\end{eqnarray*}
Therefore
\begin{eqnarray*}
\frac{O_{x_t,i}}{[\wh b_\infty^\top (\wh U^\top O) \wh A_{x_t}]_i}
& \leq &
\frac{O_{x_t,i}}{O_{x_t,i} -
(\delta_\infty + \sqrt{m}\delta_\infty\Delta_{x_t} + \sqrt{m}\Delta_{x_t})} \\
& \leq &
\frac{1}{1 -
(\delta_\infty + \sqrt{m}\delta_\infty\Delta_{x_t} + \sqrt{m}\Delta_{x_t})/\alpha} \\
& \leq &
\frac{1}{1 - \eps_1}
\ \leq \ 1 + 2\eps_1
\end{eqnarray*}
as needed.
\end{proof}
\begin{proof}[Proof of Theorem~\ref{theorem:main-tracking}]
The proof is mostly the same as that of Theorem~\ref{theorem:main-l1} with
$t=1$, except that Lemma~\ref{lemma:tracking} introduces additional error
terms.
Specifically, we require
\[
N \geq C \cdot \frac{\ln(2/\alpha)^4}{\epsilon^4 \alpha^2 \gamma^4}
\cdot \frac{m}{\sigma_m(O)^2 \sigma_m(P_{2,1})^4}
\quad \text{and} \quad
N \geq C \cdot \frac{m}{\epsilon^2 \alpha^2}
\cdot \frac{m}{\sigma_m(O)^2 \sigma_m(P_{2,1})^4}
\]
so that the terms
\[ \max\left( 4\delta_1 \log(2/\alpha), \ \sqrt{\frac{2\ln(2/\alpha)^2
\eps_0}{\gamma^2}} \right)
\quad \text{and} \quad \eps_1, \]
respectively, are $O(\epsilon)$.
The specified number of samples $N$ also suffices to imply the preconditions
of Lemma~\ref{lemma:tracking}.
The remaining terms are bounded as in the proof of
Theorem~\ref{theorem:main-l1}.
\end{proof}

\begin{lemma} \label{lemma:normalize}
If $\|\vec{a} - \vec{b}\|_1 \leq c \leq 1/2$ and $\vec{b}$ is a probability vector, then
$\|\vec{a}/(\vec{1}^\top \vec{a}) - \vec{b}\|_1 \leq 4c$.
\end{lemma}
\begin{proof}
First, it is easy to check that $1 - c \leq \vec{1}^\top \vec{a} \leq 1 + c$.
Let $I = \{ i : \vec{a}_i / (\vec{1}^\top \vec{a}) > \vec{b}_i \}$. Then for $i \in I$,
$|\vec{a}_i/(\vec{1}^\top \vec{a}) - \vec{b}_i|
= \vec{a}_i/(\vec{1}^\top \vec{a}) - \vec{b}_i
\leq \vec{a}_i/(1-c) - \vec{b}_i
\leq (1 + 2c) \vec{a}_i - \vec{b}_i
\leq |\vec{a}_i - \vec{b}_i| + 2c \vec{a}_i$.
Similarly, for $i \notin I$,
$|\vec{b}_i - \vec{a}_i/(\vec{1}^\top \vec{a})|
= \vec{b}_i - \vec{a}_i/(\vec{1}^\top \vec{a})
\leq \vec{b}_i - \vec{a}_i/(1+c)
\leq \vec{b}_i - (1-c)\vec{a}_i
\leq |\vec{b}_i - \vec{a}_i| + c \vec{a}_i$.
Therefore $\|\vec{a}/(\vec{1}^\top \vec{a}) - \vec{b}\|_1 \leq \|\vec{a} -
\vec{b}\|_1 +
2c(\vec{1}^\top \vec{a}) \leq
c + 2c(1+c) \leq 4c$.
\end{proof}

\begin{lemma} \label{lemma:kl-lb}
Let $\vec{a}$ and $\vec{b}$ be probability vectors.
If there exists some $c < 1/2$ such that $\vec{b}_i >
c$ for all $i$, then $KL(\vec{a}||\vec{b}) \leq \|\vec{a} - \vec{b}\|_1 \log (1/c)$.
\end{lemma}
\begin{proof}
See \citep{Even-DarKM07}, Lemma 3.10.
\end{proof}

\subsection*{Acknowledgments}
The authors would like to thank John Langford and Ruslan Salakhutdinov for
earlier discussions on using bottleneck methods to learn nonlinear dynamic
systems; the linearization of the bottleneck idea was the basis of this
paper.
We also thank Yishay Mansour for pointing out hardness results for learning
HMMs.
Finally, we thank Geoff Gordon, Byron Boots, and Sajid Siddiqi for
alerting us of an error in a previous version of this paper.
This work was completed while DH was an intern at TTI-C in 2008.
TZ was partially supported by the following grants: AFOSR-10097389, NSF
DMS-1007527, and NSF IIS-1016061.

\bibliography{hmm,pomdp,mixture}


\appendix

\section{Sample Complexity Bound} \label{section:sampling}

We will assume independent samples to avoid mixing estimation.
Otherwise, one can discount the number of samples by one minus the
second eigenvalue of the hidden state transition matrix $T$.

We are bounding the Frobenius norm of the matrix errors. For simplicity, we unroll the matrices into vectors, and use vector notations.

Let $z$ be a discrete random variable that takes values in
$\{1,\ldots,d\}$. We are interested in estimating the vector
$\vec{q}=[\Pr(z=j)]_{j=1}^d$ from $N$ i.i.d.~copies $z_i$ of $z$ ($i=1,\ldots,N$). 
Let $\vec{q}_i$ be the vector of zeros expect the $z_i$-th component being one.
Then the empirical estimate of $\vec{q}$ is $\wh{q}=\sum_{i=1}^N \vec{q}_i/N$.
We are interested in bounding the quantity
\[
\|\wh{q}-\vec{q}\|_2^2 .
\]

The following concentration bound is a simple application of the
McDiarmid's inequality \citep{McD89}.
\begin{proposition} \label{prop:concentration}
  We have $\forall \epsilon>0$:
  \[
  \Pr\left(\left\|\wh{q} -  \vec{q}\right\|_2 \geq 
    1/\sqrt{N} + \epsilon \right)
  \leq e^{-N \epsilon^2} .
  \]
\end{proposition}
\begin{proof}
Consider $\wh{q}=\sum_{i=1}^N \vec{q}_i/N$, and let $\wh{p}=\sum_{i=1}^N
\vec{p}_i/N$,
where $\vec{p}_i=\vec{q}_i$ except for $i=k$. Then we have
$\|\wh{q}-\vec{q}\|_2 - \|\wh{p}-\vec{q}\|_2 \leq 
\|\wh{q}-\wh{p}\|_2 \leq \sqrt{2}/N$.
By McDiarmid's inequality, we have
\[
\Pr\left(\left\|\wh{q} -  \vec{q}\right\|_2 \geq 
  \E \left\|\wh{q} -  \vec{q}\right\|_2 + \epsilon \right)
\leq e^{-N \epsilon^2} .
\]
Note that 
\begin{align*}
&\E \left\|\sum_{i=1}^N \vec{q}_i - N \vec{q} \right\|_2 
\ \leq \ \left(\E \left\|\sum_{i=1}^N \vec{q}_i- N \vec{q} \right\|_2^2 \right)^{1/2} \\
& = \ \left(\sum_{i=1}^N \E \|\vec{q}_i-\vec{q}\|_2^2 \right)^{1/2} 
\ = \
\left( \sum_{i=1}^N \E \left[1 - 2 \vec{q}_i^\top \vec{q} + \|\vec{q}\|_2^2 \right] \right)^{1/2}
\ = \ \sqrt{N (1-\|\vec{q}\|_2^2)} .
\end{align*}
This leads to the desired bound.
\end{proof}

Using this bound, we obtain with probability $1-3\eta$:
\begin{eqnarray*}
\epsilon_1 &\leq& \sqrt{\ln (1/\eta)/N} + \sqrt{1/N}, \\
\epsilon_{2,1} &\leq& \sqrt{\ln (1/\eta)/N} + \sqrt{1/N}, \\
\max_x \epsilon_{3,x,1} &\leq& \sqrt{\sum_x \epsilon_{3,x,1}^2}
\leq \sqrt{\ln (1/\eta)/N} + \sqrt{1/N}, \\
\sum_x \epsilon_{3,x,1} &\leq&
\sqrt{n} \left(\sum_x \epsilon_{3,x,1}^2\right)^{1/2} \leq \sqrt{n \ln
(1/\eta)/N} + \sqrt{n/N} .
\end{eqnarray*}

If the observation dimensionality $n$ is large and sample size
$N$ is small, then the 
third inequality can be improved by considering a more detailed estimate.
Given any $k$, let $\epsilon(k)$ be sum of elements in the
smallest $n-k$ probabilities $\Pr[x_2=x] = \sum_{i,j} [P_{3,x,1}]_{ij}$
(Equation~\ref{eq:epsk}).
Let $S_k$ be the set of these $n-k$ such $x$.
By Proposition~\ref{prop:concentration}, we obtain:
\begin{align*}
& \sum_{x \notin S_k} \|\wh{P}_{3,x,1}-P_{3,x,1}\|_F^2 + 
 \left|\sum_{x \in S_k} \sum_{i,j}
 ([\wh{P}_{3,x,1}]_{ij}-[P_{3,x,1}]_{ij})\right|^2 \\
& \leq \left( \sqrt{\ln (1/\eta)/N} + \sqrt{1/N} \right)^2.
\end{align*}
Moreover, by the definition of $S_k$, we have
\begin{align*}
\sum_{x \in S_k} \|\wh{P}_{3,x,1}-P_{3,x,1}\|_F
\leq & \sum_{x \in S_k} \sum_{i,j}  |[\wh{P}_{3,x,1}]_{ij}-[P_{3,x,1}]_{ij}| \\
\leq & \sum_{x \in S_k} \sum_{i,j}
\max\left(0,[\wh{P}_{3,x,1}]_{ij}-[P_{3,x,1}]_{ij}\right)  + \epsilon(k) \\
&\quad  + \sum_{x \in S_k} \sum_{i,j}
\min\left(0,[\wh{P}_{3,x,1}]_{ij}-[P_{3,x,1}]_{ij}\right) + \epsilon(k) \\
\leq&
 \left|\sum_{x \in S_k} \sum_{i,j} ([\wh{P}_{3,x,1}]_{ij}-[P_{3,x,1}]_{ij})\right|
+ 2 \epsilon(k) .
\end{align*}
Therefore
\[
\sum_x \epsilon_{3,x,1} \leq 
\min_{k} \left( \sqrt{k \ln (1/\eta)/N} + \sqrt{k/N} 
 + \sqrt{\ln (1/\eta)/N} + \sqrt{1/N} + 2 \epsilon(k) \right) . 
\]
This means 
$\sum_x \epsilon_{3,x,1}$ may be small even if $n$ is large, but the
number of frequently occurring symbols are small.

\section{Matrix Perturbation Theory} \label{section:matrix}

The following perturbation bounds can be found in \citep{perturbation}.

\begin{lemma}[Theorem 4.11, p.~204 in \citep{perturbation}] \label{lemma:weyl}
Let $A \in \R^{m \times n}$ with $m \geq n$, and let $\wt A = A + E$. If
the singular values of $A$ and $\wt A$ are $(\sigma_1 \geq \ldots \geq
\sigma_n)$ and $(\wt \sigma_1 \geq \ldots \geq \wt \sigma_n)$,
respectively, then
$$ |\wt \sigma_i - \sigma_i| \leq \|E\|_2 \quad i = 1, \ldots, n. $$
\end{lemma}

\begin{lemma}[Theorem 4.4, p.~262 in \citep{perturbation}] \label{lemma:svd}
Let $A \in \R^{m \times n}$, with $m \geq n$, with the singular value
decomposition $(U_1, U_2, U_3, \Sigma_1, \Sigma_2, V_1, V_2)$:
$$
\left[ \begin{array}{c} U_1^\top \\ U_2^\top \\ U_3^\top \end{array} \right]
A \left[ \begin{array}{cc} V_1 & V_2 \end{array} \right]
=
\left[ \begin{array}{cc} \Sigma_1 & 0 \\ 0 & \Sigma_2 \\ 0 & 0 \end{array}
\right].
$$
Let $\wt A = A + E$, with analogous SVD
$(\wt U_1, \wt U_2, \wt U_3, \wt \Sigma_1, \wt \Sigma_2, \wt V_1 \wt V_2)$.
Let $\Phi$ be the matrix of canonical angles between $\range(U_1)$ and
$\range(\wt U_1)$, and $\Theta$ be the matrix of canonical angles between
$\range(V_1)$ and $\range(\wt V_1)$.
If there exists $\delta, \alpha > 0$ such that
$\min \sigma(\wt \Sigma_1) \geq \alpha + \delta$ and
$\max \sigma(\Sigma_2) \leq \alpha$, then
$$
\max\{\|\sin \Phi\|_2, \|\sin \Theta\|_2\} \leq \frac{\|E\|_2}{\delta}.
$$
\end{lemma}

%
\begin{corollary} \label{cor:svd}
Let $A \in \R^{m \times n}$, with $m \geq n$, have rank $n$, and let $U \in
\R^{m \times n}$ be the matrix of $n$ left singular vectors corresponding
to the non-zero singular values $\sigma_1 \geq \ldots \geq \sigma_n > 0$ of
$A$. Let $\wt A = A + E$. Let $\wt U \in \R^{m \times n}$ be the matrix of
$n$ left singular vectors corresponding to the largest $n$ singular values
$\wt \sigma_1 \geq \ldots \geq \wt \sigma_n$ of $\wt A$, and let $\wt
U_\perp \in \R^{m \times (m-n)}$ be the remaining left singular vectors.
Assume $\|E\|_2 \leq \epsilon \sigma_n$ for some $\epsilon < 1$. Then:
\begin{enumerate}
\item $\wt \sigma_n \geq (1-\epsilon) \sigma_n$,
\item $\|\wt U_\perp^\top U\|_2 \leq \|E\|_2/\wt \sigma_n$.
\end{enumerate}
\end{corollary}
\begin{proof}
The first claim follows from Lemma~\ref{lemma:weyl}, and the second follows from
Lemma~\ref{lemma:svd} because the singular values of $\wt U_\perp^\top U$ are the
sines of the canonical angles between $\range(U)$ and $\range(\wt U)$.
\end{proof}


\begin{lemma}[Theorem 3.8, p.~143 in \citep{perturbation}] \label{lemma:pseudoinverse}
Let $A \in \R^{m \times n}$, with $m \geq n$, and let $\wt A = A + E$.
Then
$$ \| \wt A^+ - A^+ \|_2 \leq \frac{1+\sqrt{5}}{2} \cdot \max \{ \|A^+\|_2^2, \|\wt A^+\|_2^2 \} \|E\|_2. $$
\end{lemma}

\section{Recovering the Observation and Transition Matrices}
\label{section:mr06}

We sketch how to use the technique of \citep{MR06} to recover the
observation and transition matrices explicitly.
This is an extra step that can be used in conjunction with our algorithm.

Define the $n \times n$ matrix $[P_{3,1}]_{i,j} = \Pr[x_3 = i, x_1 = j]$.
Let $O_x = \diag(O_{x,1}, \ldots, O_{x,m})$, so $A_x = T O_x$.
Since $P_{3,x,1} = O A_x T \diag(\vec{\pi})O^\top$,
we have $P_{3,1} = \sum_x P_{3,x,1} = O T T \diag(\vec{\pi}) O^\top$.
Therefore
\begin{eqnarray*}
U^\top P_{3,x,1} & = & U^\top O T O_x T \diag(\vec{\pi}) O^\top \\
& = & (U^\top O T) O_x (U^\top O T)^{-1} (U^\top O T) T \diag(\vec{\pi}) O^\top \\
& = & (U^\top O T) O_x (U^\top O T)^{-1} (U^\top P_{3,1}).
\end{eqnarray*}
The matrix $U^\top P_{3,1}$ has full row rank, so $(U^\top P_{3,1}) (U^\top
P_{3,1})^+ = I$, and thus
\begin{equation*}
(U^\top P_{3,x,1}) (U^\top P_{3,1})^+ \ = \
(U^\top O T) \ O_x \ (U^\top O T)^{-1}.
\end{equation*}
Since $O_x$ is diagonal, the eigenvalues of $(U^\top P_{3,x,1}) (U^\top
P_{3,1})^+$ are exactly the observation probabilities $O_{r,1}, \ldots,
O_{r,m}$.

Define i.i.d.~random variables $g_x \sim N(0,1)$ for each $x$.
It is shown in \citep{MR06} that the eigenvalues of
\begin{equation*}
\sum_x g_x (U^\top P_{3,x,1}) (U^\top P_{3,1})^+
\ = \
(U^\top O T) \ \left(\sum_x g_x O_x\right) \ (U^\top O T)^{-1}.
\end{equation*}
will be separated with high probability (though the separation is roughly
on the same order as the failure probability; this is the main source of
instability with this method).
Therefore an eigen-decomposition will recover the columns of $(U^\top O T)$
up to a diagonal scaling matrix $S$, \emph{i.e.}~$U^\top O T S$.
Then for each $x$, we can diagonalize $(U^\top P_{3,x,1})(U^\top
P_{3,1})^+$:
\begin{equation*}
(U^\top O T S)^{-1} \ (U^\top P_{3,x,1})(U^\top P_{3,1})^+ \ (U^\top O T S)
\ = \ O_x.
\end{equation*}
Now we can form $O$ from the diagonals of $O_x$.
Since $O$ has full column rank, $O^+ O = I_m$, so it is now easy to also
recover $\vec{\pi}$ and $T$ from $P_1$ and $P_{2,1}$:
\begin{equation*}
O^+ P_1 \ = \ O^+ O \vec{\pi} \ = \ \vec{\pi}
\end{equation*}
and
\begin{equation*}
O^+ P_{2,1} (O^+)^\top \diag(\vec{\pi})^{-1}
\ = \ O^+ (O T \diag(\vec{\pi}) O^\top) (O^+)^\top \diag(\vec{\pi})^{-1}
\ = \ T.
\end{equation*}

Note that because \citep{MR06} do not allow more observations than states,
they do not need to work in a lower dimensional subspace such as
$\range(U)$.
Thus, they perform an eigen-decomposition of the matrix
$$ \sum_x g_x P_{3,x,1} P_{3,1}^{-1} = (O T) \left( \sum_x g_x O_x \right)
(O T)^{-1}, $$
and then use the eigenvectors to form the matrix $O T$.
Thus they rely on the stability of the eigenvectors, which depends heavily
on the spacing of the eigenvalues.
Consequently, the resulting sample complexity of the algorithm is
polynomial in $1/\eta$ (as opposed to $\log(1/\eta)$) where $\eta$ is the
allowed probability of failure.

\end{document}